\documentclass[letterpaper]{article} 
\usepackage{aaai25}  
\usepackage{times}  
\usepackage{helvet}  
\usepackage{courier}  
\usepackage[hyphens]{url}  
\usepackage{graphicx} 
\usepackage{natbib}  
\usepackage{caption} 
\usepackage{algorithm}
\usepackage{algorithmic}

\usepackage{amsthm}
\usepackage{pifont}
\usepackage{enumitem}
\usepackage{amsmath}
\usepackage{amssymb}
\usepackage{color}
\usepackage{xcolor}
\usepackage{booktabs}
\usepackage{multirow}
\usepackage{makecell}
\usepackage{colortbl}
\usepackage{adjustbox}
\usepackage{tikz}
\usepackage{subfigure}
\usepackage{enumitem}

\newcommand{\modelname}{DG-Mamba}
\newcommand{\ie}{\textit{i.e.}}
\newcommand{\eg}{\textit{e.g.}}

\newcommand{\etc}{\textit{etc.}}

\theoremstyle{definition}
\newtheorem{asm}{Assumption}

\newtheorem{mydef}{Definition}

\newtheorem{Prop}{Proposition}
\newtheorem{Lemma}{Lemma}

\usepackage{newfloat}
\usepackage{listings}
\DeclareCaptionStyle{ruled}{labelfont=normalfont,labelsep=colon,strut=off} 
\floatstyle{ruled}
\newfloat{listing}{tb}{lst}{}
\floatname{listing}{Listing}

\usepackage{hyperref}

\title{DG-Mamba: Robust and Efficient Dynamic Graph Structure Learning with Selective State Space Models}
\author{
    Haonan Yuan\textsuperscript{\rm 1}, Qingyun Sun\textsuperscript{\rm 1}, Zhaonan Wang\textsuperscript{\rm 2}, Xingcheng Fu\textsuperscript{\rm 3}, Cheng Ji\textsuperscript{\rm 1},\\Yongjian Wang\textsuperscript{\rm 4}, Bo Jin\textsuperscript{\rm 4}, Jianxin Li\textsuperscript{\rm 1}\thanks{Corresponding author.}
}
\affiliations{
    \textsuperscript{\rm 1}SKLCCSE, School of Computer Science and Engineering, Beihang University, China\\
    \textsuperscript{\rm 2}National Superior College for Engineers, Beihang University, China\\
    \textsuperscript{\rm 3}Key Lab of Education Blockchain and Intelligent Technology, Ministry of Education, Guangxi Normal University, China\\
    \textsuperscript{\rm 4}The Third Research Institute of Ministry of Public Security, China\\
    \{yuanhn, sunqy, wangzn, jicheng, lijx\}@buaa.edu.cn, fuxc@gxnu.edu.cn, wangyongjian@mcst.org.cn, jinbo@gass.ac.cn

}

\usepackage{bibentry}

\begin{document}

\maketitle

\begin{abstract}
Dynamic graphs exhibit intertwined spatio-temporal evolutionary patterns, widely existing in the real world. Nevertheless, the structure incompleteness, noise, and redundancy result in poor robustness for Dynamic Graph Neural Networks (DGNNs).
Dynamic Graph Structure Learning (DGSL) offers a promising way to optimize graph structures. However, aside from encountering unacceptable quadratic complexity, it overly relies on heuristic priors, making it hard to discover underlying predictive patterns.
How to efficiently refine the dynamic structures, capture intrinsic dependencies, and learn robust representations, remains under-explored.
In this work, we propose the novel \textbf{\modelname}, a robust and efficient \underline{\textbf{D}}ynamic \underline{\textbf{G}}raph structure learning framework with the Selective State Space Models (\underline{\textbf{Mamba}}).
To accelerate the spatio-temporal structure learning, we propose a kernelized dynamic message-passing operator that reduces the quadratic time complexity to linear.
To capture global intrinsic dynamics, we establish the dynamic graph as a self-contained system with State Space Model.
By discretizing the system states with the cross-snapshot graph adjacency, we enable the long-distance dependencies capturing with the selective snapshot scan. 
To endow learned dynamic structures more expressive with informativeness, we propose the self-supervised Principle of Relevant Information for DGSL to regularize the most relevant yet least redundant information, enhancing global robustness.
Extensive experiments demonstrate the superiority of the robustness and efficiency of our \modelname~compared with the state-of-the-art baselines against adversarial attacks.
\end{abstract}


\section{Introduction}
Dynamic graphs are ubiquitous in real world, spanning domains such as social media~\cite{sun2022aligning}, traffic networks~\cite{guo2021learning}, financial transactions~\cite{pareja2020evolvegcn}, and human mobility~\cite{zhou2023predicting}, \etc~Their complex spatial and temporal correlation patterns present significant challenges across various downstream deployments. Leveraging exceptional expressive capabilities, Dynamic Graph Neural Networks (DGNNs) intrinsically excel at dynamic graph representation learning by modeling both spatial and temporal predictive patterns, which enjoy the combined merits of both GNNs and sequential models.

\begin{figure}[t]
\begin{minipage}{\textwidth}
    \raggedright
    \begin{adjustbox}{width=0.485\linewidth}
    \input{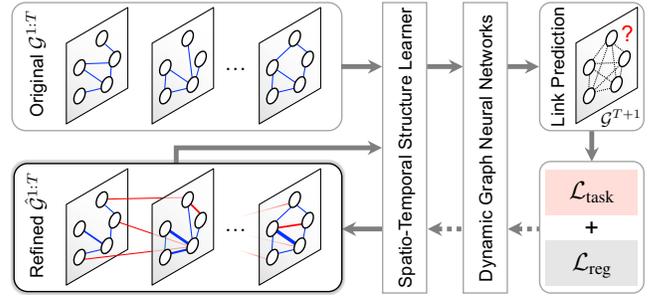}
    \end{adjustbox}
\end{minipage}
\vspace{-0.5em}
\caption{A general paradigm of DGSL.}
\label{fig:dgsl}
\vspace{-1.23em}
\end{figure}

Recently, there has been a growing research trend on enhancing the efficacy of DGNNs~\cite{zhang2022dynamic, zhu2023wingnn, yuan2024dynamic}. This includes a specific focus on improving their ability to capture the intricate spatio-temporal correlations that surpass the first-order Weisfeiler-Leman (1-WL) graph isomorphism test~\cite{XuHLJ19}. Most of the DGNNs perform spatio-temporal message-passing over potentially flawed graph structures, assuming the observed graph structures can reflect the ground-truth relationships between nodes. However, this fundamental assumption often leads to suboptimal robustness and generalization performance due to the inherent incompleteness, noise, and redundancy in graph structures, which also make the learned representations susceptible to noise and intended adversarial attacks~\cite{zugner2018adversarial}.


Graph Structure Learning (GSL) has emerged as a crucial graph learning paradigm for iteratively optimizing structures and representations~\cite{sun2022graph,sun2022position,wei2024poincare,fu2023hyperbolic}. Similarly, the GSL for dynamic graphs (DGSL) aims to refine spatial- and temporal-wise structures (Figure~\ref{fig:dgsl}). Despite its potential, DGSL faces several challenges. On one hand, the intricate coupling of spatial and temporal dimensions makes it particularly vulnerable to noise and adversarial attacks in open data environments, which significantly hampers its performance and robustness. Moreover, the predictive patterns indicate the causal behind real-world decision-making. However, current approaches often emphasize the optimization of local structures, overlooking latent long-range dependencies. This oversight results in a notable decline in performance on long-sequence dynamic graph prediction tasks as the sequence length increases.

On the other hand, DGSL demonstrates vast complexity challenges. In addition to the overwhelming quadratic complexity within individual graphs caused by node-pair probabilistic measuring~\cite{wu2022nodeformer}, attention-based sequential models (\eg, Transformer~\cite{vaswani2017attention}) also encounter quadratic complexity when computing step-pair attentions~\cite{shen2021efficient}. As the scale of nodes and sequence length grow explosively, this complexity bottleneck severely hinders the advancement of existing DGSL frameworks.
Several works have attempted to address the complexity problems from spatial and temporal perspectives~\cite{wu2022nodeformer,gu2023mamba}.
However, dynamic graphs function as a system with intricate spatio-temporal couplings, necessitating a holistic framework for complexity reduction across both dimensions.
Notably, the complexities of both dimensions are not orthogonal but interdependent. Attempting to isolate and reduce spatial complexity without considering temporal complexity, or vice versa, is insufficient. Their interdependence significantly amplifies the overall computational burden.
This multifaceted framework is essential for achieving efficient and robust dynamic graph learning, ensuring it can not only handle quadratic complexity but also capture insightful correlations.

\textbf{Research Question:} \textit{How to capture long-range intrinsic dependency of underlying predictive patterns to derive robust representations against adversarial attacks over the denoised structures while \textbf{simultaneously} reduce both spatial and temporal time complexity from quadratic to linear?}

\textbf{Present Work.}
In this work, we introduce \textbf{\modelname}, a robust and efficient \underline{\textbf{D}}ynamic \underline{\textbf{G}}raph structure learning framework with the selective state space models (\underline{\textbf{Mamba}}).
We propose a kernelized dynamic message-passing operator that reduces the quadratic time complexity to linear to accelerate the spatio-temporal structure learning. To break the local Markovian dependence assumption limitations and capture global intrinsic dynamics, we model the dynamic graph with the State Space Model as a system, and discretize the system states with the cross-snapshot graph adjacency. To endow the learned dynamic structures with informative expressiveness, we propose the self-supervised Principle of Relevant Information for DGSL to regularize the most relevant yet least redundant information, enhancing global robustness for downstream tasks. \textbf{Our contributions are:}
\begin{itemize}[leftmargin=*]
    \item We propose a robust and efficient dynamic graph structure learning framework \modelname~with linear time invariance property for robust representations against adversarial attacks. To the best of our knowledge, this is the first trial in which the spatio-temporal computational complexity of DGSL has been simultaneously reduced to linear.
    \item The kernelized dynamic graph message-passing operator behaves efficient DGSL with the help of state-discretized SSM. The structural information between the original and the learned is regularized with the proposed PRI for DGSL to enhance the robustness of the global representation. 
    \item Experiments on real-world and synthetic dynamic graphs validate the effectiveness, robustness, and efficiency of the proposed \modelname, demonstrating its superiority over 12 state-of-the-art baselines against adversarial attacks.
\end{itemize}

\section{Related Work}
\label{sec:related_work}
\subsection{Robust Dynamic Graph Learning}
Dynamic Graph Neural Networks (DGNNs) are prevalent in learning representations by inherently modeling both spatial and temporal features~\cite{han2021dynamic}. However, dynamic graphs naturally contain noise and redundant features irrelevant to the target, which compromises DGNN performance. Additionally, DGNNs are prone to the over-smoothing phenomenon, making them less robust and vulnerable to perturbations and adversarial attacks~\cite{zhu2023wingnn}.
Compared to robust GNNs for static graphs, there are no DGNNs tailored for efficient robust representation learning, currently.

\subsection{Dynamic Graph Structure Learning}
Graph Structure Learning (GSL) has gained much attention in recent years, aiming to simultaneously learn a denoised structure and robust representations~\cite{zhu2021deep}, where existing works have successfully investigated GSL methods for static graphs. However, structure learning for dynamic graphs (DGSL) remains largely under-explored, which faces the significant challenge of the computational efficiency bottleneck, as existing methods exhibit quadratic complexity in both spatial and temporal dimensions, rendering them impractical for large-scale and long-sequence dynamic graphs.

\subsection{Graph Modeling with State Space Models}
State Space Models (SSMs) are foundations for modeling dynamic systems. Recently, Mamba~\cite{gu2023mamba} has shown promising performance in efficient sequence modeling. Intuitively, there are several explorations of applying SSMs to graph modeling by converting the non-Euclidean structures to token sequence~\cite{behrouz2024graph, wang2024graph, huang2024can} but present
unique challenges due to the lack of canonical node ordering. Further, simply transforming dynamic graphs into sequences for handling by SSMs is less satisfying~\cite{behrouz2024graph, wang2024graph}, as the spatio-temporal coupling of long-range dependencies is difficult to capture, and informative feature patterns are overlooked.
\section{Preliminary}
\textbf{Notation.}
We primarily consider the discrete dynamic representation learning. A discrete dynamic graph is denoted as $\mathbf{DG} = \{ \mathcal{G}^t \}_{t=1}^{T}$, where $T$ is the time length. $\mathcal{G}^{t} = \left(\mathcal{V}^{t}, \mathcal{E}^{t}\right)$ is the graph at time $t$, where $\mathcal{V}^{t}$ is the node set and $\mathcal{E}^{t} $ is the edge set. Let $\mathbf{A}^{t} \in \{0,1\}^{N \times N}$ be the adjacency matrix and $\mathbf{X}^{t} \in \mathbb{R}^{N \times d}$ be the node features, where $N = | \mathcal{V}^t |$ denotes the number of nodes and $d$ denotes the feature dimension. 

\textbf{Dynamic Graph Representation Learning.}
As the most challenging task of dynamic graph representation learning, the future link prediction aims to train a model $f_{{\boldsymbol{\theta}}}: \mathcal{V} \times \mathcal{V} \mapsto \{0,1\}^{N \times N}$ that predicts the existence of edges at $T+1$ given historical graphs $\mathcal{G}^{1:T}$ and next-step node features $\mathbf{X}^{T+1}$. Concretely, the $f_{{\boldsymbol{\theta}}} = w\circ g$ is compound of a encoder $w(\cdot)$ and a link predictor $g(\cdot)$, \ie, $\mathbf{Z}^{T+1} = w(\mathcal{G}^{1:T}, \mathbf{X}^{T+1})$ and $\hat{\mathbf{Y}}^{T+1} = g(\mathbf{Z}^{T+1})$. The target is to iteratively learn the refined graph $\hat{\mathcal{G}}^{1:T}$ with corresponding robust dynamic graph representations for downstream tasks efficiently.
\begin{figure*}[!t]
\begin{minipage}{\textwidth}
    \hspace{-0.2cm}
    \begin{adjustbox}{width=1.015\linewidth}
    \input{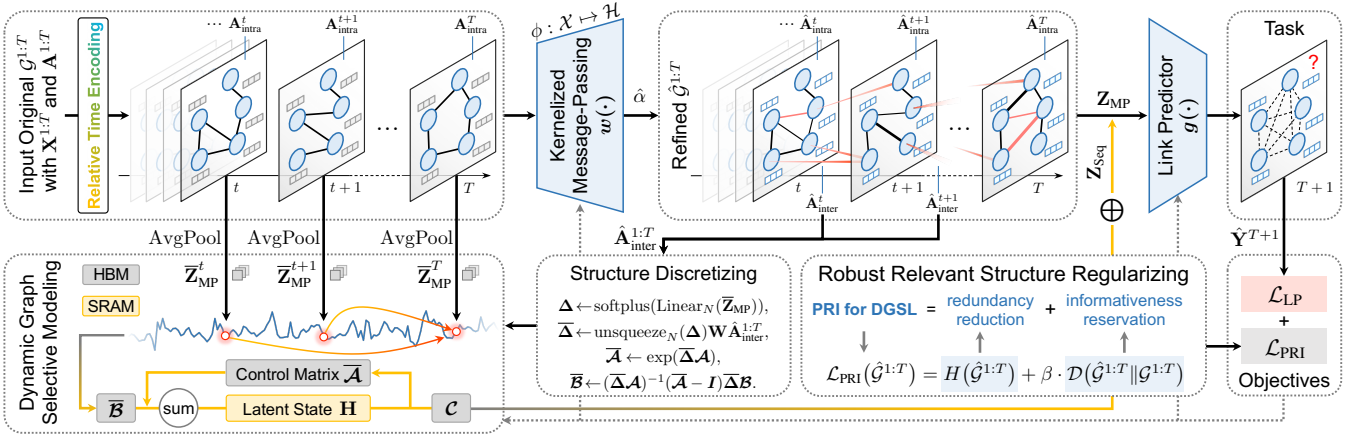}
    \end{adjustbox}
\end{minipage}
\vspace{-0.35cm}
\caption{The framework of \modelname. (a) Kernelized message-passing mechanism learns both intra- and inter-graph weights with linear time complexity. (b) Long-range dependencies are strengthened by selective modeling with parameters discretized by learned inter-graph structures. (c) PRI for DGSL is proposed to guarantee robustness against noise and adversarial attacks.}
\label{fig:framework}
\vspace{-0.305cm}
\end{figure*}
\section{\modelname: Robust and Efficient\\Dynamic Graph Structure Learning}
This section elaborates on \modelname~with its framework shown in Figure~\ref{fig:framework}.
First, we propose a kernelized dynamic graph message-passing operator to accelerate both the spatial and temporal structure learning.
Then, we model and discretize the dynamic graph system with inter-graph structures to capture long-range dependencies and intrinsic dynamics while retaining sequential linear complexity.
Lastly, we promote the robustness of representations by the self-supervised Principle of Relevant Information, trading off between the most relevant yet least redundant structural information.

\subsection{Kernelized Message-Passing for Efficient Dynamic Graph Structure Learning}
\label{sec:kernel}
To efficiently learn spatio-temporal structures, we propose the kernelized message-passing mechanism performing on a dynamic graph attention network, where the learnable edge weights play a role in both structure refinement and attentive feature aggregation.
As most literature presumed, we make the following assumption.

\begin{asm}[\textbf{Dynamic Graph Markov Dependence}]
\label{asm:markov}
    Assume the $\mathbf{DG} = \{ \mathcal{G}^t \}_{t=1}^{T}$ follows the Markov Chain: $\langle\mathcal{G}^1 \to \cdots \to \mathcal{G}^T\rangle$. Given graph $\mathcal{G}^t$ at present, the next-step graph $\mathcal{G}^{t+1}$ is conditionally independent of the past $\mathcal{G}^{<t}$, \ie,
\begin{equation}
    \mathbb{P}(\mathcal{G}^{t+1} \mid \mathcal{G}^{1:t}) = \mathbb{P}(\mathcal{G}^{t+1} \mid \mathcal{G}^{t}).
\end{equation}
\end{asm}
Assumption~\ref{asm:markov} declares the local dependencies that sculpt the weighted message-passing routes between graph pairs. Given node $u$ in $\mathcal{G}^t$ at the $l$-th layer, the attentive aggregation at the next ($l+$1)-th layer is,
\begin{equation}
    \mathbf{z}_u^{t(l+1)} = \sum_v \hat{\alpha}_{uv}^{t-1:t(l)}(\mathbf{W}\mathbf{z}_v^{\cdot(l)}), \text{for all}~v \!\in\! \mathcal{N}(u)^{t-1:t},
\label{eq:update}
\end{equation}
where $\mathbf{W}$ is learnable matrix. $\mathcal{N}(u)^{t-1:t}$ denotes $u$'s neighbors in $\mathcal{G}^{t-1}$ and $\mathcal{G}^t$. $\hat{\alpha}_{uv}^{t-1:t(l)}$ contains structure weights for $\mathcal{G}^t$ and message-passing routes between $\mathcal{G}^{t-1}$ and $\mathcal{G}^t$, \ie,
\begin{equation}
    \hat{\alpha}_{uv}^{t-1:t(l)}=\frac{\exp(\sigma((\mathbf{W}\mathbf{z}_u^{t(l)})^\top(\mathbf{W}\mathbf{z}_v^{\cdot(l)})))}{\sum_m \exp(\sigma((\mathbf{W}\mathbf{z}_u^{t(l)})^\top(\mathbf{W}\mathbf{z}_m^{\cdot(l)})))}.
\label{eq:strcuture_learning}
\end{equation}
Note that, we omit the limits of the summations always for nodes in the $\mathcal{N}(u)^{t-1:t}$ for brevity.
Intuitively, Softmax pair-wise edge weights updating and representation aggregation in Eq.~\eqref{eq:update} and Eq.~\eqref{eq:strcuture_learning} contribute to unacceptable quadratic complexity for dynamic graph structure learning. Inspired by kernel-based methods that employ kernel functions to model edge weights~\cite{zhu2021deep}, we combine Eq.~\eqref{eq:update} and Eq.~\eqref{eq:strcuture_learning} with kernel $k(\cdot,\cdot)$ for measuring similarity, \ie,
\begin{equation}
    \mathbf{z}_u^{t(l+1)} = \sum_v \frac{k(\mathbf{W}\mathbf{z}_u^{t(l)}, \mathbf{W}\mathbf{z}_v^{\cdot(l)})}{\sum_m k(\mathbf{W}\mathbf{z}_u^{t(l)}, \mathbf{W}\mathbf{z}_m^{\cdot(l)})} \cdot \mathbf{W}\mathbf{z}_v^{\cdot(l)}.
\label{eq:update_2}
\end{equation}

\textbf{Kernel Estimation by Random Features.}
Instead of explicitly finding a feature map $\varphi$ from representation space $\mathcal{X}$ to the reproducing kernel Hilbert space $\mathcal{H}$ and calculate the kernel $k(\cdot,\cdot)$ by inner production $\langle\cdot,\cdot\rangle_\mathcal{H}$, the Mercer's theorem guarantees an implicitly defined function $\phi$ exists if and only if $k(\cdot,\cdot)$ is a positive definite kernel~\cite{mercer1909xvi}, \ie,
\begin{equation}
    k(\mathbf{x}_1,\mathbf{x}_2)_\mathcal{X} = \langle\varphi(\mathbf{x}_1), \varphi(\mathbf{x}_2)\rangle_\mathcal{H} \doteq \phi(\mathbf{x}_1)^\top \phi(\mathbf{x}_2).
\end{equation}
In this way, Eq.~\eqref{eq:update_2} can be converted into a simpler form,
\begin{align}
    \mathbf{z}_u^{t(l+1)} &= \sum_v \frac{\phi(\mathbf{W}\mathbf{z}_u^{t(l)})^\top\phi(\mathbf{W}\mathbf{z}_v^{\cdot(l)})}{\sum_m \phi(\mathbf{W}\mathbf{z}_u^{t(l)})^\top\phi(\mathbf{W}\mathbf{z}_m^{\cdot(l)})}\cdot \mathbf{W}\mathbf{z}_v^{\cdot(l)}\\
    &=\frac{\phi(\mathbf{W}\mathbf{z}_u^{t(l)})^\top\sum_v\phi(\mathbf{W}\mathbf{z}_v^{\cdot(l)})\mathbf{W}\mathbf{z}_v^{\cdot(l)\top}}{\phi(\mathbf{W}\mathbf{z}_u^{t(l)})^\top\sum_m\phi(\mathbf{W}\mathbf{z}_m^{\cdot(l)})}.
\label{eq:update_3}
\end{align}
Note that, $\phi(\mathbf{W}\mathbf{z}_u^{t(l)})^\top$ is irreducible as it is the matrices operations. It is noteworthy that the two summations greatly contribute to decreasing the quadratic complexity as they can be computed once and stored for each $u$. Intuitively, kernel $k(\cdot,\cdot)$ can be estimated by the Positive Random Features (PRF)~\cite{choromanski2020rethinking} for Softmax approximation in Lamma 1 that satisfies the Mercer's theorem, \ie,
\begin{equation}
    \phi(\mathbf{x}) = \sum_{i=1}^m \frac{1}{\sqrt{m}}\exp\left(\boldsymbol{\omega}_i^\top\mathbf{x}-\frac{\|\mathbf{x}\|_2^2}{2}\right),
\label{eq:prf}
\end{equation}
where $m$ is the projection dimension of kernel $k(\cdot,\cdot)$, and $\boldsymbol{\omega}_i$ is the random feature shifting to the target embedding space. Under such settings, the structure updating and representation aggregation are differentiable by the Gumbel-Softmax reparameterization trick~\cite{wu2022nodeformer}. We integrate it to Eq.~\eqref{eq:update_3} with approximation error proofs in Appendix~\ref{sec:proof_1}.

\textbf{Intra- and Inter-Structure Efficient Query.}
As Eq.~\eqref{eq:update_3} merges both the edge reweighting and message-passing process in a unified and implicit manner, we can still explicitly obtain the optimized edge weights $\hat{\alpha}_{uv}^{t-1:t}$ by efficiently querying the approximated kernels of any node pairs with details in Appendix \ref{sec:proof_1}. We decompose overall $\hat{\alpha}_{uv}^{t-1:t}$ into intra-graph and inter-graph weights, which constructed two types of adjacency matrices for the consequent refining, \ie,
\begin{align}
    \hat{\mathbf{A}}^{t}_\text{intra} &= \{\hat{\alpha}_{uv}^{t-1:t}\}^{N\times N}, \text{where}~u,v\in \mathcal{V}^t,\label{eq:intra_structure}\\
    \hat{\mathbf{A}}^{t}_\text{inter} &= \{\hat{\alpha}_{uv}^{t-1:t}\}^{N\times N}, \text{where}~u\in \mathcal{V}^t, v\in \mathcal{V}^{t-1},\label{eq:inter_structure}
\end{align}
where $\hat{\mathbf{A}}^{t}_\text{intra}$ and $\hat{\mathbf{A}}^{t}_\text{inter}$ are then for the intra-graph and inter-graph structure regularizing (Eq.~\eqref{eq:structure_pri}), respectively.

\subsection{Long-range Dependencies Selective Modeling}
\label{sec:ssm}
Though the kernelized spatio-temporal message-passing has reduced the quadratic complexity, this success is contingent upon Assumption~\ref{asm:markov}, which substantially compromises with the Markov condition. However, real-world dynamic graphs can be exceedingly long and exhibit uncertain periodic variations, characterized by long-range dependencies between graph snapshots, where the local dependencies constraints significantly hinder their selective feature capturing.

\textbf{Dynamic Graph System Modeling.}
To strengthen the global long-range dependencies selective modeling without increasing the spatial computational complexity, we propose constructing the dynamic graph as a self-contained system with the State Space Models. Specifically, SSMs are defined with the state transition matrix $\boldsymbol{\mathcal{A}} \in \mathbb{R}^{n\times n}$ and two projection matrices $\boldsymbol{\mathcal{B}} \in \mathbb{R}^{n\times 1}$, $\boldsymbol{\mathcal{C}} \in \mathbb{R}^{1\times n}$. Given the continuous input sequence $\mathbf{x}(t) \in \mathbb{R}^l$, the SSM updates the latent state $\mathbf{h}(t) \in \mathbb{R}^{n\times l}$ and output $\mathbf{y}(t) \in \mathbb{R}^l$, \ie,
\begin{equation}
    \mathbf{h}^\prime(t) = \boldsymbol{\mathcal{A}}\mathbf{h}(t) + \boldsymbol{\mathcal{B}}\mathbf{x}(t), \;\mathbf{y}(t) = \boldsymbol{\mathcal{C}}\mathbf{h}(t),
\label{eq:ssm}
\end{equation}
where $\boldsymbol{\mathcal{A}}$ controls how current state evolves over time in a global view, $\boldsymbol{\mathcal{B}}$ describes how the input influences the state, and $\boldsymbol{\mathcal{C}}$ responses how the current state translate to the output.

To effectively integrate Eq.~\eqref{eq:ssm} within the deep learning settings, it is essential to discretize the continuous system. However, there are two critical problems to address: How to enable SSMs attention-aware to each time step in replacing the self-attention mechanism that consumes quadratic complexity? And how we incorporate the refined inter-graph structures into the state updating process such that weighted message-passing routes between graphs can be considered?

\textbf{Selective Discretizing and Parameterizing with Inter-Graph Structures.}
To address the aforementioned problems, we propose the Dynamic Graph Selective Scan Mechanism that discretizes the system parameters ($\boldsymbol{\mathcal{A}}$, $\boldsymbol{\mathcal{B}}$, $\boldsymbol{\mathcal{C}}$, \etc) function of each step input to selectively control which part of the graph sequence with how much attention can flow into the hidden state, and parameterized system states with the inter-graph structures $\hat{\mathbf{A}}^{1:T}_\text{inter}$ to integrate local dependencies into the long-range global dependencies capturing.

Denote $\mathbf{Z}_\text{MP} \in \mathbb{R}^{B\times T \times N\times D_0}$ as the learned node embeddings after spatio-temporal message-passing once in Eq.~\eqref{eq:update_3}, where $B$ means the batch size, $T$ denotes the dynamic graph length, and $D_0$ represents the latent feature dimension. Such that, the sequential input $\overline{\mathbf{Z}}_\text{MP} \in \mathbb{R}^{B\times T\times N}$ is the average pooling on $\mathbf{Z}_\text{MP}$, which implies the current latent state for each step. $\boldsymbol{\mathcal{A}} \in \mathbb{R}^{N\times D}$ is randomly initialized. $\boldsymbol{\mathcal{B}}$ and $\boldsymbol{\mathcal{C}}$ are further parameterized to each step input, \ie, 
\begin{equation}
    \boldsymbol{\mathcal{B}}, \boldsymbol{\mathcal{C}} \in \mathbb{R}^{B\times T \times D} \gets \operatorname{Linear}_{D}(\overline{\mathbf{Z}}_\text{MP}).
\end{equation}
Additionally, a timestep-wise parameter $\boldsymbol{\Delta} \in \mathbb{R}^{B\times T \times N}$ initialized by the input $\overline{\mathbf{Z}}_\text{MP}$ is utilized to discrete the dynamic graph system with the learned inter-graph structures, \ie,
\begin{align}
    &\boldsymbol{\Delta} \in \mathbb{R}^{B\times T \times N} \gets \operatorname{softplus}(\operatorname{Linear}_{N}(\overline{\mathbf{Z}}_\text{MP})),\\
    &\overline{\boldsymbol{\Delta}} \in \mathbb{R}^{B\times T \times N} \gets \operatorname{unsqueeze}_N(\boldsymbol{\Delta})\cdot \mathbf{W} \hat{\mathbf{A}}^{1:T}_\text{inter}.
\end{align}
Following the continuous signal reconstructing zero-order hold (ZOH) rules, parameters $\boldsymbol{\mathcal{A}}$ and $\boldsymbol{\mathcal{B}}$ are discretized by,
\begin{equation}
    \overline{\boldsymbol{\mathcal{A}}}\!\gets\! \exp(\overline{\boldsymbol{\Delta}}\boldsymbol{\mathcal{A}}),\overline{\boldsymbol{\mathcal{B}}}\!\gets\!(\overline{\boldsymbol{\Delta}} \boldsymbol{\mathcal{A}})^{-1}(\exp (\overline{\boldsymbol{\Delta}}\boldsymbol{\mathcal{A}})-\boldsymbol{I})\overline{\boldsymbol{\Delta}}\boldsymbol{\mathcal{B}}.
\end{equation}
Consequently, the output of the dynamic graph system is,
\begin{equation}
    \mathbf{H}_t = \overline{\boldsymbol{\mathcal{A}}} \mathbf{H}_{t-1} + \overline{\boldsymbol{\mathcal{B}}} (\overline{\mathbf{Z}}_\text{MP})_t, \;(\mathbf{Z}_\text{Seq})_t = \boldsymbol{\mathcal{C}}\mathbf{H}_t,
\label{eq:conv}
\end{equation}
where $\mathbf{H}$ is the latent states. The detailed dynamic graph selective scan is described in line 2 to line 11 in Algorithm~\ref{alg:alg_selective}. The output $\mathbf{Z}_\text{Seq} \in \mathbb{R}^{B\times T\times N}$ selectively emerged the temporal semantics with its long-range dependencies, which is then acting as the supervision on $\mathbf{Z}_\text{MP}$, leading to the combined node representation, \ie,
\vspace{-0.05cm}
\begin{equation}
    \hat{\mathbf{Z}} = \mathbf{Z}_\text{MP} + \lambda\cdot\operatorname{unsqueeze}_{D_0}(\mathbf{Z}_\text{Seq}),
\label{eq:merge}
\vspace{-0.05cm}
\end{equation}
where $\lambda$ is the hyperparameter. The convolution update in Eq.~\eqref{eq:conv} can scale up linearly in length with the hardware-aware algorithm described in Appendix~\ref{sec:hardware}.

\subsection{Robust Relevant Structure Regularizing}
\label{sec:pri}
The last milestone for the research goal is to strengthen the robustness of the updated representations against potential noise and adversarial attacks in surrounding environments.
This is a dual-purpose objective: while the learned implicit structure represents intrinsic dependencies, the physically-structured raw graphs contain rich interpretability semantics. Robust DGSL should expect to learn minimal but sufficient structural information from an information-theoretic view.



\textbf{Principle of Relevant Information (PRI).}
To reduce redundant structural information as well as reserve critical predictive patterns, we utilize the self-supervised PRI~\cite{principe2010information} to formulate the criteria for dynamic graph structure learning, which plays the role of structural regularizers.
\begin{mydef}[\textbf{PRI for DGSL}]
\label{def:pri}
    Given dynamic graph $\mathcal{G}^{1:T}$, the Principle of Relevant Information for DGSL aims to regularize the refined graph $\hat{\mathcal{G}}^{1:T}$ by,
    \begin{equation}
        \mathcal{L}_\text{PRI}(\hat{\mathcal{G}}^{1:T}) = H(\hat{\mathcal{G}}^{1:T}) + \beta\cdot\mathcal{D}(\hat{\mathcal{G}}^{1:T}\|\mathcal{G}^{1:T}),
    \label{eq:pri}
    \end{equation}
    where $H(\cdot)$ denotes the Shannon entropy that measures the redundancy of $\hat{\mathcal{G}}^{1:T}$. $\mathcal{D}(\cdot\|\cdot)$ is the divergence that reflects the discrepancy between two terms. The hyperparameter $\beta$ plays the trade-off between the redundancy reduction and predictive patterns reservation. Larger $\beta$ leads to more information reserved from the input dynamic graphs, and vice versa.
\end{mydef}
\noindent PRI for DGSL is indispensable for strengthening structure robustness in a self-supervised manner, for it encourages \modelname~emphasizes on the informative, discriminative, and invariant structural patterns across historical graph snapshots while filtering out potential noise and redundant information that damages the parameter fitting process.

\textbf{Derivation of PRI for DGSL.}
As optimize Eq.~\eqref{eq:pri} straightforwardly is indifferentiable, we approximately decompose it into respective spatial- and temporal-wise regularizing with learned intra- and inter-graph structures, \ie,
\begin{equation}
    \mathcal{L}_\text{PRI}(\hat{\mathcal{G}}^{1:T}) \triangleq \mathcal{L}_\text{PRI}(\hat{\mathbf{A}}^{1:T}_\text{intra}) + \mathcal{L}_\text{PRI}(\hat{\mathbf{A}}^{1:T}_\text{inter}).
\label{eq:structure_pri}
\end{equation}
To regularize intra-graph structures, we transform the divergence term into edge-level constraints with the loss equivalence guarantee in Appendix~\ref{sec:proof_2}, \ie,
\begin{align}
    &\mathcal{L}_\text{PRI}(\hat{\mathbf{A}}^{1:T}_\text{intra}) \doteq H(\hat{\mathbf{A}}^{1:T}_\text{intra}) +\beta_1\cdot \mathcal{L}_\text{edge},\label{eq:L_intra}\\
    \text{and}\quad&\mathcal{L}_\text{edge} = -\frac{1}{NT}\sum_{t=1}^T \sum_{u,v\in\mathcal{E}^t} \frac{1}{d(u)}\log \hat{\alpha}^t_{uv},\label{eq:L_edge}
\end{align}
where $\beta_1$ is the hyperparameter, and $d(\cdot)$ measures the node degree. Eq.~\eqref{eq:L_edge} is the maximum likelihood estimation for edges in $\mathcal{E}^t$. For inter-graph structure regularizing, as there is no feasible gound-truth supervision for the original ${\mathbf{A}}^{1:T}_\text{inter}$, we utilize the structure-aware $\mathbf{Z}_\text{Seq}$ and $\overline{\mathbf{Z}}_\text{MP}$ instead, \ie,
\begin{equation}
    \mathcal{L}_\text{PRI}(\hat{\mathbf{A}}^{1:T}_\text{inter}) \doteq 
    H(\mathbf{Z}_\text{Seq}) + \beta_2\cdot\mathcal{D}(\mathbf{Z}_\text{Seq}\|\overline{\mathbf{Z}}_\text{MP}),
\label{eq:L_inter}
\end{equation}
where $\beta_2$ is a hyperparameter, and the KL-divergence is implemented to the divergence term $\mathcal{D}(\cdot,\cdot)$.

\subsection{Optimization and Complexity Analysis}
The overall optimization objective of \modelname~is,
\begin{equation}
    \mathcal{L} = \mathcal{L}_\text{LP}(\mathbf{Y}^{T+1},\hat{\mathbf{Y}}^{T+1}) + \mu\cdot\mathcal{L}_\text{PRI}(\hat{\mathcal{G}}^{1:T}),
\label{eq:loss_all}
\end{equation}
where $\mathcal{L}_\text{LP}$ is implemented by the cross-entropy loss for future link prediction, $\mathcal{L}_\text{PRI}(\hat{\mathcal{G}}^{1:T})$ is derived by Eq.~\eqref{eq:L_intra} and Eq.~\eqref{eq:L_inter}. $\mu$ is the Lagrangian hyperparameter. The training pipeline is illustrated in Algorithm~\ref{alg:alg_overall} and \ref{alg:alg_selective} (Appendix~\ref{sec:alg}).

\subsubsection{Computational Complexity.}
We denote $|\mathcal{V}|$ and $|\mathcal{E}|$ as the average number of nodes and edges in each graph snapshot, respectively, and $T$ denotes the graph length.
The computational complexity of the kernelized message-passing (Eq.~\eqref{eq:update_3}) is $\mathcal{O}(T|\mathcal{V}|)$, and the intra- (Eq.~\eqref{eq:intra_structure}) and inter-graph structure query (Eq.~\eqref{eq:inter_structure}) contributes $\mathcal{O}(T|\mathcal{E}|)$. For dynamic graph selective scan, we implement a hardware-aware algorithm to accelerate the long-range dependencies modeling by the kernel fusion and recomputation~\cite{gu2023mamba}, which approximately requires $\mathcal{O}(T)$. We omit the computational complexity brought by feature projection and aggregation for brevity as the feature dimensions are significantly smaller than $|\mathcal{V}|$ and $|\mathcal{E}|$. Such that, the overall computational complexity of \modelname~is linear with the length, averaged number of nodes, and edges, \ie,
\begin{equation}
    \mathcal{O}(T(|\mathcal{V}|+|\mathcal{E}|)),
\end{equation}
which is superior efficient than state-of-the-art DGNNs, especially when the original graphs are less dense. Detailed complexity analysis can be found in Appendix~\ref{sec:alg}.
\begin{table*}[t]
  \centering
  \setlength{\tabcolsep}{3pt}
  \hspace{-0.17cm}
  \resizebox{1.001\textwidth}{!}{
    \begin{tabular}{lccccccccccccccc}
    \toprule
    \textbf{Dataset} & \multicolumn{5}{c}{\textbf{COLLAB}}           & \multicolumn{5}{c}{\textbf{Yelp}} & \multicolumn{5}{c}{\textbf{ACT}} \\
    \midrule
    \multirow{2}[2]{*}{\textbf{Model}} & \multirow{2}[2]{*}{\textbf{Clean}} & \multirow{2}[2]{*}{\makecell{\textbf{Structure}\\\textbf{Attack}}} & \multicolumn{3}{c}{\textbf{Feature Attack}} & \multirow{2}[2]{*}{\textbf{Clean}} & \multirow{2}[2]{*}{\makecell{\textbf{Structure}\\\textbf{Attack}}} & \multicolumn{3}{c}{\textbf{Feature Attack}} & \multirow{2}[2]{*}{\textbf{Clean}} & \multirow{2}[2]{*}{\makecell{\textbf{Structure}\\\textbf{Attack}}} & \multicolumn{3}{c}{\textbf{Feature Attack}} \\
\cmidrule{4-6}\cmidrule{9-11}\cmidrule{14-16}          &      &       & $\lambda=\text{0.5}$ & $\lambda=\text{1.0}$ & $\lambda=\text{1.5}$ &       &       & $\lambda=\text{0.5}$ & $\lambda=\text{1.0}$ & $\lambda=\text{1.5}$ &       &       & $\lambda=\text{0.5}$ & $\lambda=\text{1.0}$ & $\lambda=\text{1.5}$ \\
\midrule
    GAE   & 77.15\scalebox{0.75}{±0.5} & 74.04\scalebox{0.75}{±0.8} & 50.59\scalebox{0.75}{±0.8} & 44.66\scalebox{0.75}{±0.8} & 43.12\scalebox{0.75}{±0.8} & 70.67\scalebox{0.75}{±1.1} & 64.45\scalebox{0.75}{±5.0} & 51.05\scalebox{0.75}{±0.6} & 45.41\scalebox{0.75}{±0.6} & 41.56\scalebox{0.75}{±0.9} & 72.31\scalebox{0.75}{±0.5} & 60.27\scalebox{0.75}{±0.4} & 56.56\scalebox{0.75}{±0.5} & 52.52\scalebox{0.75}{±0.6} & 50.36\scalebox{0.75}{±0.9} \\
    VGAE  & 86.47\scalebox{0.75}{±0.0} & 74.95\scalebox{0.75}{±1.2} & 56.75\scalebox{0.75}{±0.6} & 50.39\scalebox{0.75}{±0.7} & 48.68\scalebox{0.75}{±0.7} & 76.54\scalebox{0.75}{±0.5} & 65.33\scalebox{0.75}{±1.4} & 55.53\scalebox{0.75}{±0.7} & 49.88\scalebox{0.75}{±0.8} & 45.08\scalebox{0.75}{±0.6} & 79.18\scalebox{0.75}{±0.5} & 66.29\scalebox{0.75}{±1.3} & 60.67\scalebox{0.75}{±0.7} & 57.39\scalebox{0.75}{±0.8} & 55.27\scalebox{0.75}{±1.0} \\
    GAT   & 88.26\scalebox{0.75}{±0.4} & 77.29\scalebox{0.75}{±1.8} & 58.13\scalebox{0.75}{±0.9} & 51.41\scalebox{0.75}{±0.9} & 49.77\scalebox{0.75}{±0.9} & 77.93\scalebox{0.75}{±0.1} & 69.35\scalebox{0.75}{±1.6} & 56.72\scalebox{0.75}{±0.3} & 52.51\scalebox{0.75}{±0.5} & 46.21\scalebox{0.75}{±0.5} & 85.07\scalebox{0.75}{±0.3} & 77.55\scalebox{0.75}{±1.2} & 66.05\scalebox{0.75}{±0.4} & 61.85\scalebox{0.75}{±0.3} & 59.05\scalebox{0.75}{±0.3} \\
    \midrule
    GCRN & 82.78\scalebox{0.75}{±0.5} & 69.72\scalebox{0.75}{±0.5} & 54.07\scalebox{0.75}{±0.9} & 47.78\scalebox{0.75}{±0.8} & 46.18\scalebox{0.75}{±0.9} & 68.59\scalebox{0.75}{±1.0} & 54.68\scalebox{0.75}{±7.6} & 52.68\scalebox{0.75}{±0.6} & 46.85\scalebox{0.75}{±0.6} & 40.45\scalebox{0.75}{±0.6} & 76.28\scalebox{0.75}{±0.5} & 64.35\scalebox{0.75}{±1.2} & 59.48\scalebox{0.75}{±0.7} & 54.16\scalebox{0.75}{±0.6} & 53.88\scalebox{0.75}{±0.7} \\
    EvolveGCN & 86.62\scalebox{0.75}{±1.0} & 76.15\scalebox{0.75}{±0.9} & 56.82\scalebox{0.75}{±1.2} & 50.33\scalebox{0.75}{±1.0} & 48.55\scalebox{0.75}{±1.0} & 78.21\scalebox{0.75}{±0.0} & 53.82\scalebox{0.75}{±2.0} & 57.91\scalebox{0.75}{±0.5} & 51.82\scalebox{0.75}{±0.3} & 45.32\scalebox{0.75}{±1.0} & 74.55\scalebox{0.75}{±0.3} & 63.17\scalebox{0.75}{±1.0} & 61.02\scalebox{0.75}{±0.5} & 53.34\scalebox{0.75}{±0.5} & 51.62\scalebox{0.75}{±0.7} \\
    DySAT & 88.77\scalebox{0.75}{±0.2} & 76.59\scalebox{0.75}{±0.2} & 58.28\scalebox{0.75}{±0.3} & 51.52\scalebox{0.75}{±0.3} & 49.32\scalebox{0.75}{±0.5} & 78.87\scalebox{0.75}{±0.6} & 66.09\scalebox{0.75}{±1.4} & 58.46\scalebox{0.75}{±0.4} & 52.33\scalebox{0.75}{±0.7} & 46.24\scalebox{0.75}{±0.7} & 78.52\scalebox{0.75}{±0.4} & 66.55\scalebox{0.75}{±1.2} & 61.94\scalebox{0.75}{±0.8} & 56.98\scalebox{0.75}{±0.8} & 54.14\scalebox{0.75}{±0.7} \\
    SpoT-Mamba & 84.34\scalebox{0.75}{±0.4} & 74.39\scalebox{0.75}{±0.2} & 54.76\scalebox{0.75}{±0.8} & 48.64\scalebox{0.75}{±0.9} & 47.25\scalebox{0.75}{±0.7} & 77.01\scalebox{0.75}{±1.0} & 60.56\scalebox{0.75}{±1.2} & 54.72\scalebox{0.75}{±0.8} & 50.11\scalebox{0.75}{±0.8} & 44.95\scalebox{0.75}{±0.8} & 73.29\scalebox{0.75}{±1.0} & 61.27\scalebox{0.75}{±0.9} & 59.92\scalebox{0.75}{±0.7} & 52.19\scalebox{0.75}{±0.8} & 51.33\scalebox{0.75}{±0.9} \\
    \midrule
    RDGSL & 82.29\scalebox{0.75}{±0.5} & 71.36\scalebox{0.75}{±0.9} & 52.33\scalebox{0.75}{±0.5} & 48.50\scalebox{0.75}{±0.7} & 45.21\scalebox{0.75}{±0.6} & 75.92\scalebox{0.75}{±0.6} & 58.30\scalebox{0.75}{±0.9} & 52.29\scalebox{0.75}{±0.5} & 48.66\scalebox{0.75}{±0.4} & 44.59\scalebox{0.75}{±0.5} & 73.15\scalebox{0.75}{±0.6} & 62.45\scalebox{0.75}{±1.0} & 60.14\scalebox{0.75}{±0.6} & 53.05\scalebox{0.75}{±0.5} & 51.07\scalebox{0.75}{±0.5} \\
    TGSL & 84.09\scalebox{0.75}{±0.5} & 73.66\scalebox{0.75}{±1.0} & 55.29\scalebox{0.75}{±0.4} & 51.34\scalebox{0.75}{±0.4} & 50.28\scalebox{0.75}{±0.3} & 76.55\scalebox{0.75}{±0.4} & 73.29\scalebox{0.75}{±1.1} & 60.21\scalebox{0.75}{±0.3} & 51.01\scalebox{0.75}{±0.3} & 49.87\scalebox{0.75}{±0.4} & 80.53\scalebox{0.75}{±0.5} & 70.32\scalebox{0.75}{±0.9} & 67.19\scalebox{0.75}{±0.4} & 60.27\scalebox{0.75}{±0.5} & 58.39\scalebox{0.75}{±0.5} \\
    \midrule
    RGCN  & 88.21\scalebox{0.75}{±0.1} & 78.66\scalebox{0.75}{±0.7} & 61.29\scalebox{0.75}{±0.5} & 54.29\scalebox{0.75}{±0.6} & 52.99\scalebox{0.75}{±0.6} & 77.28\scalebox{0.75}{±0.3} & 74.29\scalebox{0.75}{±0.4} & 59.72\scalebox{0.75}{±0.3} & 52.88\scalebox{0.75}{±0.3} & 50.40\scalebox{0.75}{±0.2} & 87.22\scalebox{0.75}{±0.2} & 82.66\scalebox{0.75}{±0.4} & 68.51\scalebox{0.75}{±0.2} & 62.67\scalebox{0.75}{±0.2} & 61.31\scalebox{0.75}{±0.2} \\
    WinGNN & 90.33\scalebox{0.75}{±0.1} & 82.34\scalebox{0.75}{±0.6} & \underline{64.69\scalebox{0.75}{±0.9}} & 56.87\scalebox{0.75}{±1.1} & 54.44\scalebox{0.75}{±0.6} & 76.46\scalebox{0.75}{±1.0} & 74.59\scalebox{0.75}{±0.8} & 60.45\scalebox{0.75}{±0.4} & 55.80\scalebox{0.75}{±1.0} & 52.73\scalebox{0.75}{±0.8} & 90.12\scalebox{0.75}{±0.4} & 85.36\scalebox{0.75}{±0.4} & 71.60\scalebox{0.75}{±0.9} & 65.40\scalebox{0.75}{±0.3} & 63.32\scalebox{0.75}{±0.8} \\
    DGIB-Bern & 92.17\scalebox{0.75}{±0.2} & 83.58\scalebox{0.75}{±0.1} & 63.54\scalebox{0.75}{±0.9} & 56.92\scalebox{0.75}{±1.0} & \underline{56.24\scalebox{0.75}{±1.0}} & 76.88\scalebox{0.75}{±0.2} & 75.61\scalebox{0.75}{±0.0} & \textbf{63.91\scalebox{0.75}{±0.9}} & \textbf{59.28\scalebox{0.75}{±0.9}} & \underline{54.77\scalebox{0.75}{±1.0}} & 94.49\scalebox{0.75}{±0.2} & 87.75\scalebox{0.75}{±0.1} & 73.05\scalebox{0.75}{±0.9} & 68.49\scalebox{0.75}{±0.9} & \underline{66.27\scalebox{0.75}{±0.9}} \\
    DGIB-Cat  & \underline{92.68\scalebox{0.75}{±0.1}} & \underline{84.16\scalebox{0.75}{±0.1}} & 63.99\scalebox{0.75}{±0.5} & \underline{57.76\scalebox{0.75}{±0.8}} & 55.63\scalebox{0.75}{±1.0} & \underline{79.53\scalebox{0.75}{±0.2}} & \textbf{77.72\scalebox{0.75}{±0.1}} & 61.42\scalebox{0.75}{±0.9} & 55.12\scalebox{0.75}{±0.7} & 51.90\scalebox{0.75}{±0.9} & \underline{94.89\scalebox{0.75}{±0.2}} & \underline{88.27\scalebox{0.75}{±0.2}} & \underline{73.92\scalebox{0.75}{±0.8}} & \underline{68.88\scalebox{0.75}{±0.9}} & 65.99\scalebox{0.75}{±0.7} \\
    \midrule 
    \textbf{\modelname} & \textbf{93.60\scalebox{0.75}{±0.3}} & \textbf{92.60\scalebox{0.75}{±0.3}} & \textbf{68.53\scalebox{0.75}{±1.5}} & \textbf{60.88\scalebox{0.75}{±1.0}} & \textbf{56.95\scalebox{0.75}{±0.8}} & \textbf{81.54\scalebox{0.75}{±0.6}} & \underline{77.40\scalebox{0.75}{±0.7}} & \underline{61.82\scalebox{0.75}{±0.9}} & \underline{57.42\scalebox{0.75}{±0.6}} & \textbf{55.97\scalebox{0.75}{±1.2}} & \textbf{96.67\scalebox{0.75}{±0.3}} & \textbf{96.14\scalebox{0.75}{±0.3}} & \textbf{79.36\scalebox{0.75}{±0.8}} & \textbf{73.76\scalebox{0.75}{±0.7}} & \textbf{70.21\scalebox{0.75}{±0.7}} \\
    \bottomrule
    \end{tabular}%
    }
    \vspace{-0.75em}
    \caption{AUC score (\% ± standard deviation for five runs) of the future link prediction task on real-world datasets against \textbf{non-targeted} (random) adversarial attacks. The best results are shown in \textbf{bold} type and the runner-ups are \underline{underlined}.}
  \label{tab:non-targeted}%
\end{table*}%

\section{Experiment}
\label{sec:exp}
In this section, we conduct extensive experiments on real-world and synthetic dynamic graph datasets to evaluate the effectiveness, robustness, efficiency of \modelname\footnote{\texttt{\url{https://github.com/RingBDStack/DG-Mamba}}.} against multi-type adversarial attacks. Detailed settings and additional results can be found in Appendix~\ref{sec:exp_detail} and Appendix~\ref{sec:more_results}.

\subsection{Experimental Settings}

\subsubsection{Dynamic Graph Datasets.}
\label{sec:data}
We evaluate \modelname~on the challenging future link prediction with three real-world dynamic graph datasets. \ding{172} \textbf{COLLAB} \cite{tang2012cross} is an academic collaboration dataset with papers published in 16 years. \ding{173} \textbf{Yelp}~\cite{sankar2020dysat} contains customer reviews for 24 months. \ding{174} \textbf{ACT}~\cite{kumar2019predicting} describes actions taken by users on a popular MOOC website within 30 days, and each action has a binary label. Statistics of the datasets are concluded in Table~\ref{tab:datasets}.

\subsubsection{Baselines.}
We compare with four categories, 12 baselines.\\
\ding{172}~\textbf{Static GNNs:}
GAE and VGAE \cite{kipf2016variational}, GAT \cite{velivckovic2018graph}.
\ding{173}~\textbf{DGNNs:}
GCRN \cite{seo2018structured}, EvolveGCN \cite{pareja2020evolvegcn}, DySAT \cite{sankar2020dysat}, and SpoT-Mamba \cite{choi2024spot}.
\ding{174}~\textbf{DGSL:}
RDGSL \cite{zhang2023rdgsl}, TGSL \cite{zhang2023time}.
\ding{175}~\textbf{Robust (D)GNNs:}
RGCN \cite{zhu2019robust}, WinGNN \cite{zhu2023wingnn}, and DGIB \cite{yuan2024dynamic}.

\subsubsection{Adversarial Attack Settings.}
We compare baselines and \modelname~under two typical adversarial attack scenarios.

\begin{itemize}[leftmargin=*]
    \item \textbf{Non-targeted:}
    We make synthetic datasets by attacking graph structures and node features, respectively.
    \ding{172} \textbf{Structure Attack:} We randomly remove one out of five types of edges in the training and validation graphs. This removal makes the task more challenging than the real-world situations as the model cannot access any information on the removed edges. \ding{173} \textbf{Feature Attack:} Gaussian noise $\lambda \cdot r \cdot \epsilon$ is added to the node features, where $r$ is the reference amplitude of the original features, and $\epsilon \sim N(\mathbf{0}$, $\mathbf{I})$. Parameter $\lambda \in$ \{0.5, 1.0, 1.5\} controls the degree of the attack.
    \item \textbf{Targeted:}
    We apply the prevailing NETTACK~\cite{zugner2018adversarial}, a targeted adversarial attack library on graphs designed to target nodes by altering their connected edges or node features. We simultaneously consider the evasion and poisoning attack.
    \ding{172}~\textbf{Evasion Attack:} Train on clean data, test on the attacked data. \ding{173} \textbf{Poisoning Attack:} The entire dataset is attacked before model training and testing.
    In both scenarios, we use GAT~\cite{velivckovic2018graph} as the surrogate model. The number of perturbations $n$ is set to \{1, 2, 3\}.
\end{itemize}

\subsubsection{Parameter Settings.}
\label{sec:para_setting}
We set number of layers as 2 for baselines, 1 for \modelname~to avoid overfitting. Latent dimension is set to 128. Baseline hyperparameters follow recommended values from their papers and are fine-tuned for fairness. Configuration files provide values for $\beta_1$, $\beta_2$, $\lambda$, and $\mu$. We optimize with Adam~\cite{kingma2014adam}, selecting the learning rate from \{1e-02, 1e-03, 1e-04, 1e-05\}. The maximum number of epochs is 1,000, with early stopping.

\begin{table*}[t]
  \centering
    \setlength{\tabcolsep}{4.8pt}
    \hspace{-0.17cm}
    \resizebox{1.001\textwidth}{!}{
    \begin{tabular}{llccccccccccc}
    \toprule
    \multirow{2}[2]{*}{{\textbf{Dataset}}} & \multirow{2}[2]{*}{{\textbf{Model}}} & \multirow{2}[2]{*}{{\textbf{Clean}}} & \multicolumn{4}{c}{\textbf{Evasion Attack}} & \multicolumn{4}{c}{\textbf{Poisoning Attack}} \\
    \cmidrule(r){4-7} \cmidrule{8-11}      &       &       & $n = \text{1}$ ($\Delta\% \downarrow$) & $n = \text{2}$ ($\Delta\% \downarrow$) & $n = \text{3}$ ($\Delta\% \downarrow$) & \textbf{Avg.} $\boldsymbol{\Delta\% \downarrow}$ & $n = \text{1}$ ($\Delta\% \downarrow$) & $n = \text{2}$ ($\Delta\% \downarrow$) & $n = \text{3}$ ($\Delta\% \downarrow$) & \textbf{Avg.} $\boldsymbol{\Delta\% \downarrow}$ \\
    \midrule
\multirow{7}[4]{*}{\textbf{COLLAB}} 
           & GAT    & 88.26\scalebox{0.75}{±0.4} & 76.21\scalebox{0.75}{±0.1} (13.7) & 66.56\scalebox{0.75}{±0.1} (24.6) & 57.92\scalebox{0.75}{±0.1} (34.4) & 24.2 & 66.59\scalebox{0.75}{±0.5} (24.6) & 55.31\scalebox{0.75}{±0.6} (37.3) & 51.34\scalebox{0.75}{±0.7} (41.8) & 34.6 \\
           & DySAT  & 88.77\scalebox{0.75}{±0.2} & 77.91\scalebox{0.75}{±0.1} (12.2) & 68.22\scalebox{0.75}{±0.1} (23.1) & 58.82\scalebox{0.75}{±0.1} (33.7) & 23.0 & 69.02\scalebox{0.75}{±0.3} (22.2) & 57.62\scalebox{0.75}{±0.3} (35.1) & 52.76\scalebox{0.75}{±0.3} (40.6) & 32.6 \\
           & SpoT-Mamba & 84.34\scalebox{0.75}{±0.4} & 71.45\scalebox{0.75}{±0.2} (15.3) & 65.88\scalebox{0.75}{±0.2} (21.9) & 52.14\scalebox{0.75}{±0.3} (38.2) & 25.1 & 66.45\scalebox{0.75}{±0.5} (21.2) & 55.36\scalebox{0.75}{±0.9} (34.4) & 53.17\scalebox{0.75}{±0.6} (37.0) & 30.8 \\
           & TGSL & 84.09\scalebox{0.75}{±0.5} & 72.09\scalebox{0.75}{±0.3} (14.3) & 65.30\scalebox{0.75}{±0.2} (22.3) & 52.09\scalebox{0.75}{±0.3} (38.1) & 24.9 & 66.57\scalebox{0.75}{±0.3} (20.8) & 54.21\scalebox{0.75}{±0.2} (35.5) & 55.36\scalebox{0.75}{±0.3} (34.2) & \underline{30.2} \\
           & WinGNN & 90.33\scalebox{0.75}{±0.1} & 79.35\scalebox{0.75}{±0.2} (12.2) & 68.24\scalebox{0.75}{±0.1} (24.5) & 61.07\scalebox{0.75}{±0.3} (32.4) & 23.0 & 71.53\scalebox{0.75}{±0.8} (20.8) & \underline{61.57\scalebox{0.75}{±1.1}} (31.8) & 55.27\scalebox{0.75}{±1.0} (38.8) & 30.5 \\
           & {DGIB-Cat}  & \underline{92.68\scalebox{0.75}{±0.1}} & \underline{81.29\scalebox{0.75}{±0.0}} (12.3) & \underline{71.32\scalebox{0.75}{±0.1}} (23.0) & \underline{62.03\scalebox{0.75}{±0.1}} (33.1) & \underline{22.8} & \underline{72.55\scalebox{0.75}{±0.2}} (21.7) & 60.99\scalebox{0.75}{±0.3} (34.2) & \underline{55.62\scalebox{0.75}{±0.4}} (40.0) & 32.0 \\
           & \textbf{\modelname} & \textbf{93.60\scalebox{0.75}{±0.3}} & \textbf{81.78\scalebox{0.75}{±0.6}} (12.6) & \textbf{80.87\scalebox{0.75}{±0.6}} (13.6) & \textbf{68.75\scalebox{0.75}{±1.3}} (26.5) & \textbf{17.6} & \textbf{79.48\scalebox{0.75}{±0.2}} (15.1) & \textbf{67.45\scalebox{0.75}{±0.1}} (27.9) & \textbf{64.99\scalebox{0.75}{±0.6}} (30.6) & \textbf{24.5} \\
    \midrule
    \multirow{7}[4]{*}{\textbf{Yelp}}  
           & GAT   & 77.93\scalebox{0.75}{±0.1} & 67.96\scalebox{0.75}{±0.1} (12.8) & 59.47\scalebox{0.75}{±0.1} (23.7) & 50.27\scalebox{0.75}{±0.1} (35.5) & 24.0 & 65.34\scalebox{0.75}{±0.5} (16.2) & 54.51\scalebox{0.75}{±0.2} (30.1) & 50.24\scalebox{0.75}{±0.4} (35.5) & 27.2 \\
           & DySAT  & 78.87\scalebox{0.75}{±0.6} & 69.77\scalebox{0.75}{±0.1} (11.5) & 60.66\scalebox{0.75}{±0.1} (23.1) & 52.16\scalebox{0.75}{±0.1} (33.9) & 22.8 & 66.87\scalebox{0.75}{±0.6} (15.2) & 56.31\scalebox{0.75}{±0.3} (28.6) & 50.44\scalebox{0.75}{±0.6} (36.0) & 26.6 \\
           & SpoT-Mamba & 77.01\scalebox{0.75}{±1.0} & 65.25\scalebox{0.75}{±0.2} (15.3) & 54.33\scalebox{0.75}{±0.2} (29.5) & 47.75\scalebox{0.75}{±0.2} (38.0) & 27.6 & 64.39\scalebox{0.75}{±1.0} (16.4) & 55.21\scalebox{0.75}{±0.9} (28.3) & 50.33\scalebox{0.75}{±1.1} (34.6) & 26.4 \\
           & TGSL & 76.55\scalebox{0.75}{±0.4} & 65.03\scalebox{0.75}{±0.3} (15.0) & 54.29\scalebox{0.75}{±0.3} (29.1) &  47.81\scalebox{0.75}{±0.3} (37.5) & 27.2 & 64.08\scalebox{0.75}{±0.8} (16.3) & 56.27\scalebox{0.75}{±0.6} (26.5) & 51.20\scalebox{0.75}{±0.8} (33.1) & 25.3 \\
           & WinGNN & 76.46\scalebox{0.75}{±1.0} & 66.25\scalebox{0.75}{±1.0} (13.4) & 60.22\scalebox{0.75}{±0.9} (21.2) & \underline{51.38\scalebox{0.75}{±0.8}} (32.8) & 22.5 & \underline{67.88\scalebox{0.75}{±0.9}} (11.2) & 56.36\scalebox{0.75}{±0.9} (26.3) & \underline{52.74\scalebox{0.75}{±1.0}} (31.0) & \underline{22.8} \\
           & {DGIB-Cat}  & \underline{79.53\scalebox{0.75}{±0.2}} & \underline{70.17\scalebox{0.75}{±0.0}} (11.8) & \underline{62.25\scalebox{0.75}{±0.1}} (21.7) & \textbf{52.69\scalebox{0.75}{±0.1}} (33.7) & \underline{22.4} & 67.38\scalebox{0.75}{±0.3} (15.3) & \underline{57.02\scalebox{0.75}{±0.2}} (28.3) & 51.39\scalebox{0.75}{±0.2} (35.4) & 26.3 \\
           & \textbf{\modelname} & \textbf{81.54\scalebox{0.75}{±0.6}} & \textbf{70.88\scalebox{0.75}{±0.3}} (13.1) & \textbf{69.77\scalebox{0.75}{±0.5}} (14.4) & 49.93\scalebox{0.75}{±0.6} (38.8) & \textbf{22.1} & \textbf{73.10\scalebox{0.75}{±0.4}} (10.4) & \textbf{64.65\scalebox{0.75}{±0.1}} (20.7) & \textbf{54.67\scalebox{0.75}{±0.3}} (33.0) & \textbf{21.3} \\
    \midrule
    \multirow{7}[4]{*}{\textbf{ACT}}  
           & GAT    & 85.07\scalebox{0.75}{±0.3} & 75.14\scalebox{0.75}{±0.1} (11.7) & 67.25\scalebox{0.75}{±0.1} (20.9) & 59.75\scalebox{0.75}{±0.1} (29.8) & 20.8 & 71.26\scalebox{0.75}{±0.9} (16.2) & 61.43\scalebox{0.75}{±1.1} (27.8) & 57.35\scalebox{0.75}{±1.1} (32.6) & 25.5 \\
           & DySAT  & 78.52\scalebox{0.75}{±0.4} & 70.64\scalebox{0.75}{±0.1} (10.0) & 63.35\scalebox{0.75}{±0.0} (19.3) & 56.36\scalebox{0.75}{±0.0} (28.2) & 19.2 & 66.21\scalebox{0.75}{±0.9} (15.7) & 56.28\scalebox{0.75}{±0.9} (28.3) & 53.45\scalebox{0.75}{±1.1} (31.9) & 25.3 \\
           & SpoT-Mamba & 73.29\scalebox{0.75}{±1.0} & 65.64\scalebox{0.75}{±1.1} (10.4) & 61.99\scalebox{0.75}{±0.9} (15.4) & 51.08\scalebox{0.75}{±0.8} (30.3) & 18.7 & 62.89\scalebox{0.75}{±0.9} (14.9) & 58.04\scalebox{0.75}{±1.3} (20.8) & 51.04\scalebox{0.75}{±1.2} (30.4) & \underline{22.0} \\
           & TGSL & 80.53\scalebox{0.75}{±0.5} & 72.26\scalebox{0.75}{±0.3} (12.8) & 67.34\scalebox{0.75}{±0.3} (16.4) & 61.55\scalebox{0.75}{±0.3} (23.6) & \underline{17.6} & 68.10\scalebox{0.75}{±0.9} (15.4) & 61.07\scalebox{0.75}{±1.0} (24.2) & 59.39\scalebox{0.75}{±1.0} (26.3) & \underline{22.0} \\
           & WinGNN & 90.12\scalebox{0.75}{±0.4} & 80.16\scalebox{0.75}{±0.4} (11.1) & 72.50\scalebox{0.75}{±0.3} (19.6) & 63.21\scalebox{0.75}{±0.4} (29.9) & 20.2 & \underline{81.26\scalebox{0.75}{±0.9}} \;\;(9.8) & 67.33\scalebox{0.75}{±1.1} (25.3) & 61.25\scalebox{0.75}{±1.0} (32.0) & 22.4 \\
           & {DGIB-Cat}  & \underline{94.89\scalebox{0.75}{±0.2}} & \underline{84.98\scalebox{0.75}{±0.1}} (10.4) & \underline{76.78\scalebox{0.75}{±0.1}} (19.1) & \textbf{67.69\scalebox{0.75}{±0.1}} (28.7) & 19.4 & 80.16\scalebox{0.75}{±0.4} (15.5) & \underline{68.71\scalebox{0.75}{±0.5}} (27.6) & \underline{64.38\scalebox{0.75}{±0.6}} (32.2) & 25.1\\
           & \textbf{\modelname} & \textbf{96.67\scalebox{0.75}{±0.3}} & \textbf{86.62\scalebox{0.75}{±0.1}} (10.4) & \textbf{85.58\scalebox{0.75}{±0.1}} (11.5) & \underline{67.12\scalebox{0.75}{±0.5}} (30.6) & \textbf{17.5} & \textbf{85.53\scalebox{0.75}{±0.6}} (11.5) & \textbf{75.62\scalebox{0.75}{±0.2}} (21.8) & \textbf{65.65\scalebox{0.75}{±0.5}} (32.1) & \textbf{21.8} \\
    \bottomrule
    \end{tabular}%
   }
    \vspace{-0.75em}
     \caption{AUC score (\% ± standard deviation for five runs) of the future link prediction task on real-world datasets against \textbf{targeted} adversarial attacks. The best results are shown in \textbf{bold} type and the runner-ups are \underline{underlined}. ``$\Delta\%$'' indicates the relative performance decrease after targeted adversarial attacks compared to that on the clean datasets.}
  \label{tab:targeted}%
\vspace{-1em}
\end{table*}%
\subsection{Against Non-Targeted Adversarial Attacks}
\label{sec:exp_non_targeted}
In this section, we evaluate the model performance on future link prediction and its robustness to non-targeted (random) adversarial attacks on structures and node features. We report results using AUC (\%) scores from five runs in Table~\ref{tab:non-targeted}.

\subsubsection{Analysis.}
In most cases, \modelname~outperforms other baselines significantly.
Static GNNs are not well-suited for dynamic scenarios, struggling to adapt when structures and features evolve. Dynamic GNNs underperform due to their insufficient handling of complex coupling dynamics. DGSL baselines exhibit sensitivity to noise caused by lacking modeling of the underlying predictive patterns, leading to sharp drops in performance under high-intensity feature attacks, especially in datasets with strong temporal relations. Though DGIB slightly surpasses \modelname~in a few cases, it generally fails due to drawbacks brought by its strong Markov condition assumption, which greatly damages capturing of the long-range dependencies for strengthening robustness.

\subsection{Against Targeted Adversarial Attacks}
\label{sec:exp_targeted}
We continue to evaluate with competitive baselines standing out in Table~\ref{tab:non-targeted}, focusing on link prediction performance and defense against targeted adversarial attacks with a relative decrease. Results of AUC (\%) are reported in Table~\ref{tab:targeted}.

\subsubsection{Analysis.}
\modelname~consistently demonstrates strong robustness across all datasets compared to other competitive baselines, showing the lowest average percentage decrease under both evasion and poisoning attacks. While baselines like WinGNN and DGIB also demonstrate relative robustness, they are more affected by these sophisticated attacks, particularly in the ACT dataset, which appears to be the most challenging for most baselines with larger drops in the AUC scores across the board. Larger AUC decreases witnessed in baselines like TGSL, DySAT, and especially SpoT-Mamba generally exhibits the highest vulnerability among the baselines considered, further highlighting the challenge of defending against targeted adversarial attacks in the real world.

\begin{figure*}[!t]
    \begin{minipage}{0.51\textwidth}
        \centering
        \includegraphics[height=3.05cm]{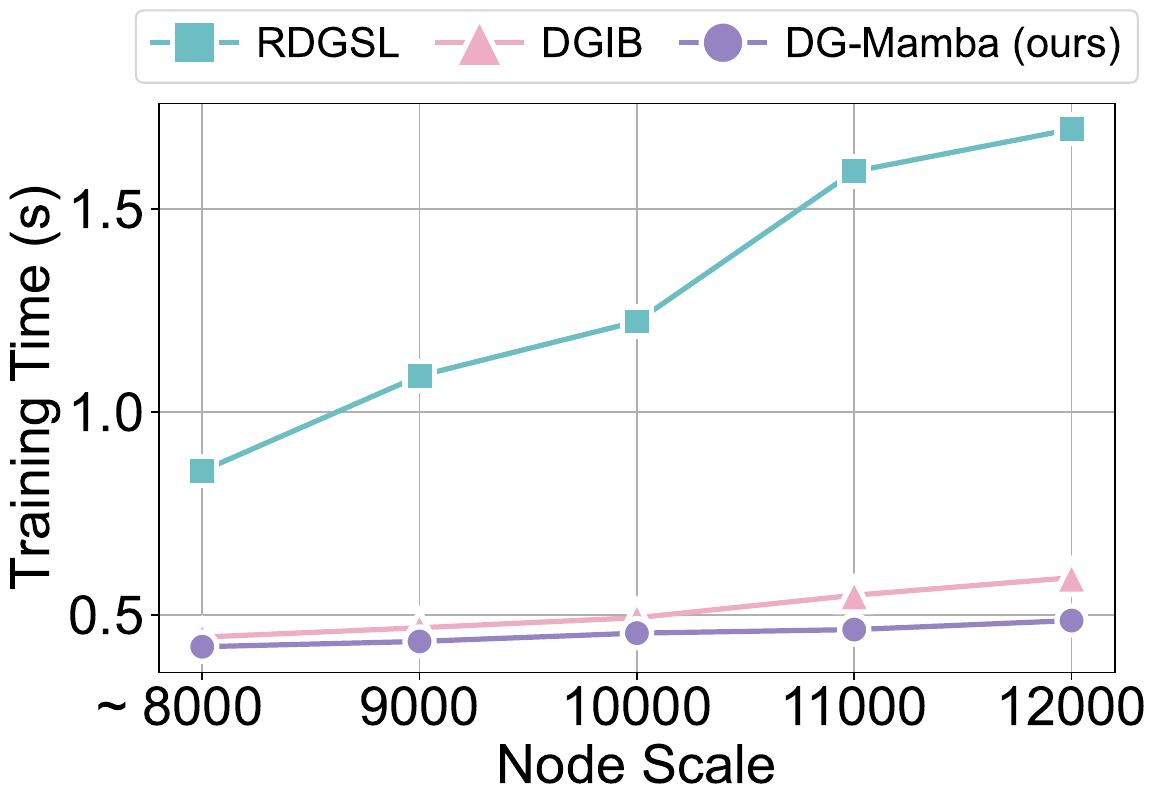}
        \hspace{-0.5em}
        \includegraphics[height=3.05cm]{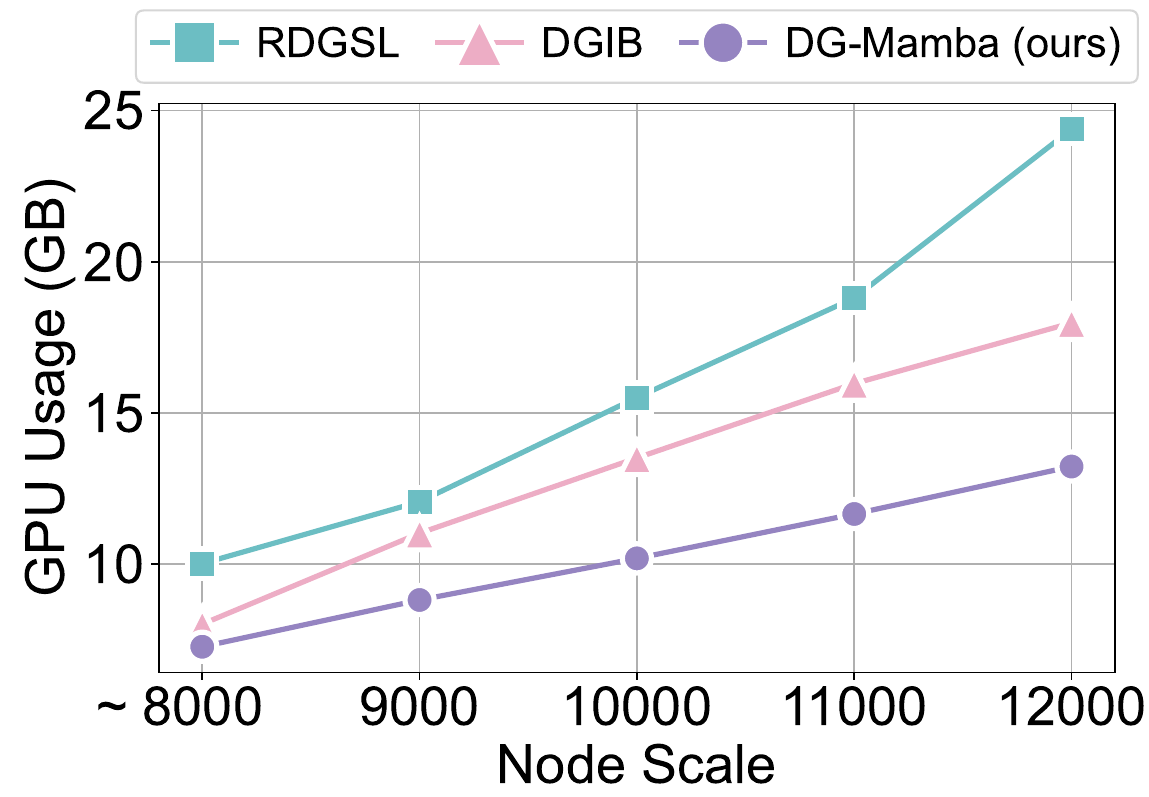}
        \vspace{-0.5em}
        \caption{Node scaling efficiency evaluation on Yelp.}
        \label{fig:effi_node_yelp}
    \end{minipage}
    \hspace{-1.1em}
    \begin{minipage}{0.51\textwidth}
        \centering
        \includegraphics[height=3.05cm]{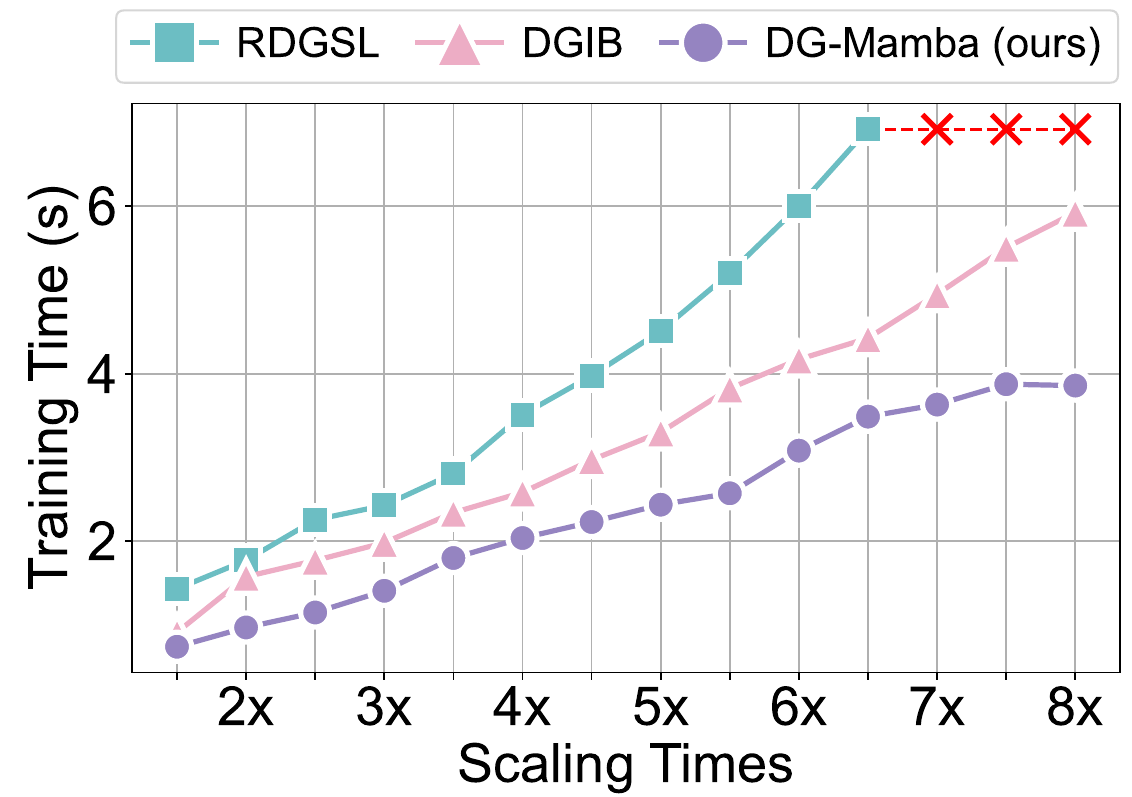}
        \hspace{-0.3em}
        \includegraphics[height=3.05cm]{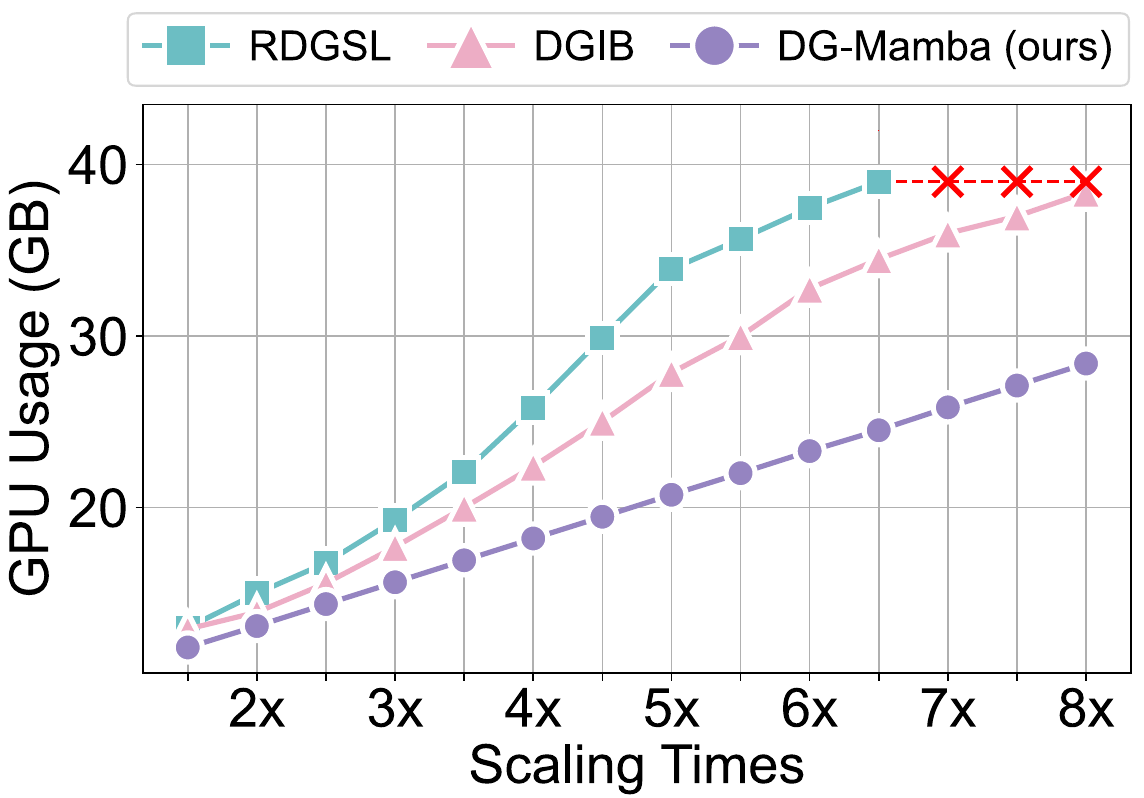}
        \vspace{-0.5em}
        \caption{Sequence scaling efficiency evaluation on Yelp.}
        \label{fig:effi_length_yelp}
    \end{minipage}
\end{figure*}

\begin{figure*}[!t]
    \begin{minipage}{0.33\textwidth}
        \centering
        \includegraphics[height=4.75cm]{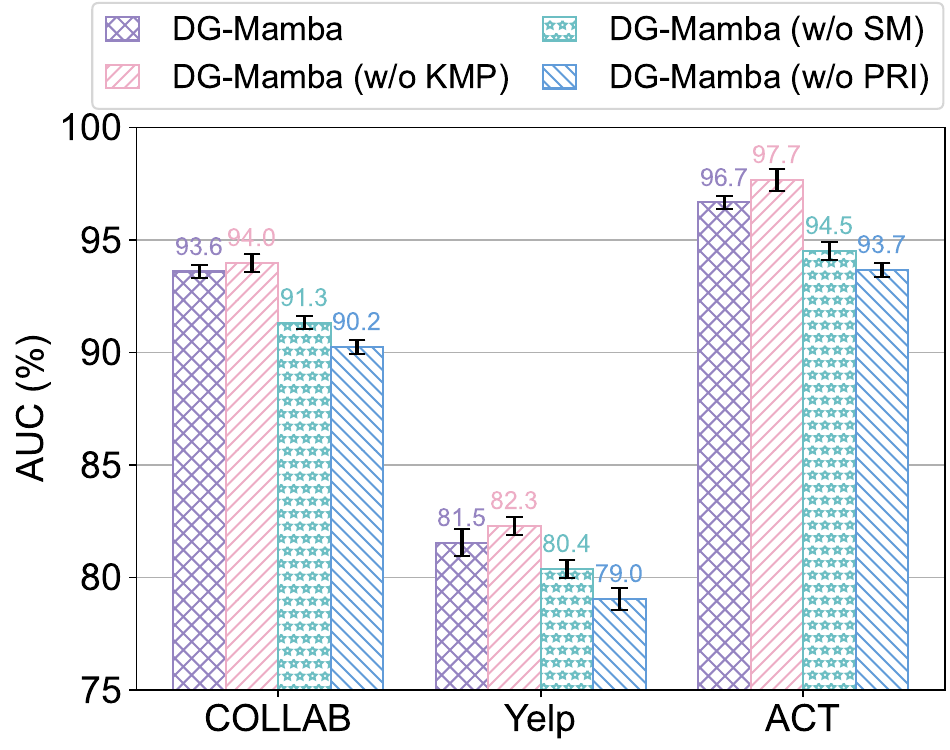}
        \vspace{-1.5em}
        \caption{Ablation study.}
        \label{fig:abla_main}
    \end{minipage}
    \hspace{0.4em}
    \begin{minipage}{0.33\textwidth}
        \centering
        \includegraphics[height=4.75cm]{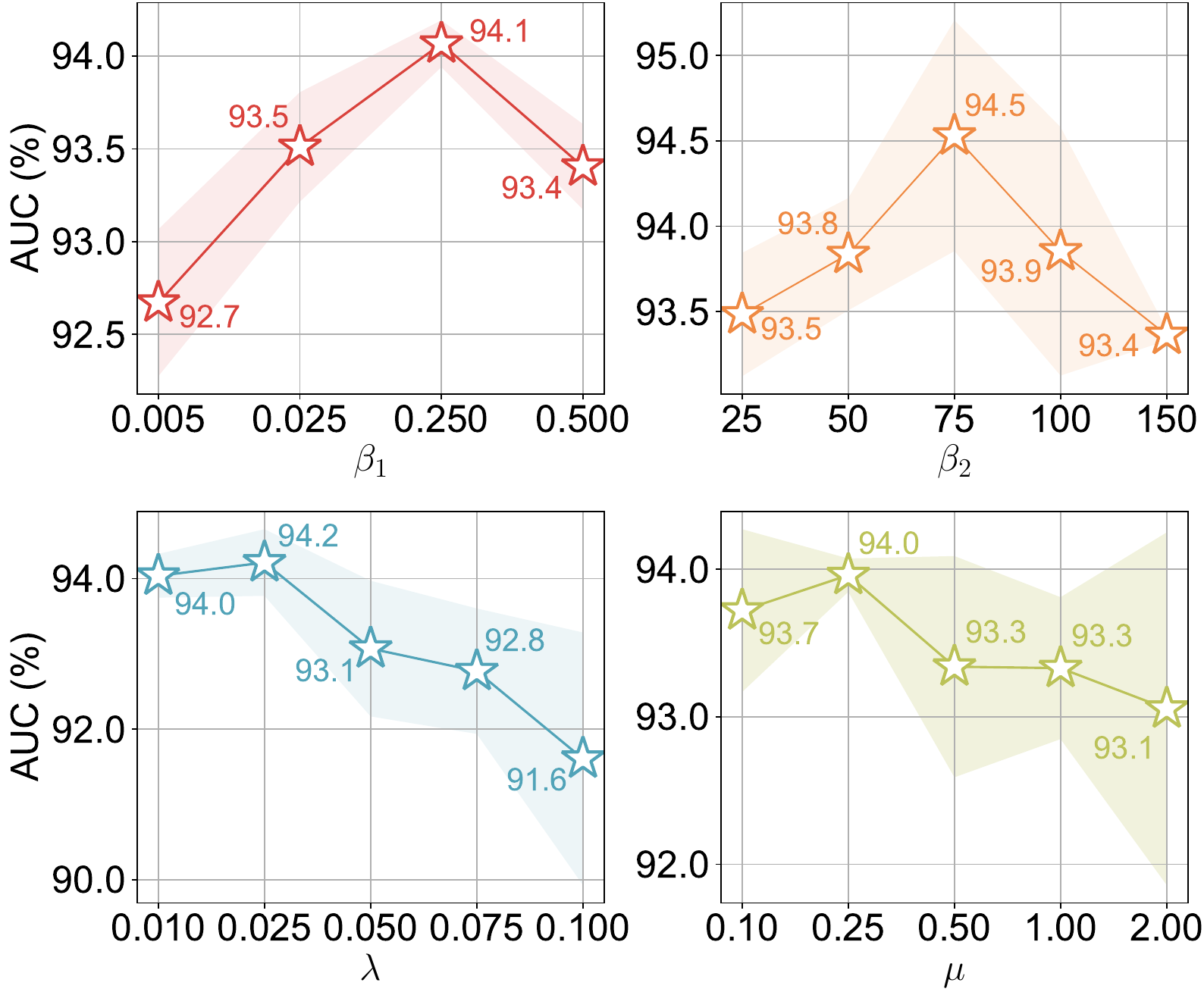}
        \vspace{-0.5em}
        \caption{Hyperparameter analysis.}
        \label{fig:hyper_main}
    \end{minipage}
    \begin{minipage}{0.33\textwidth}
        \centering
        \includegraphics[height=4.75cm]{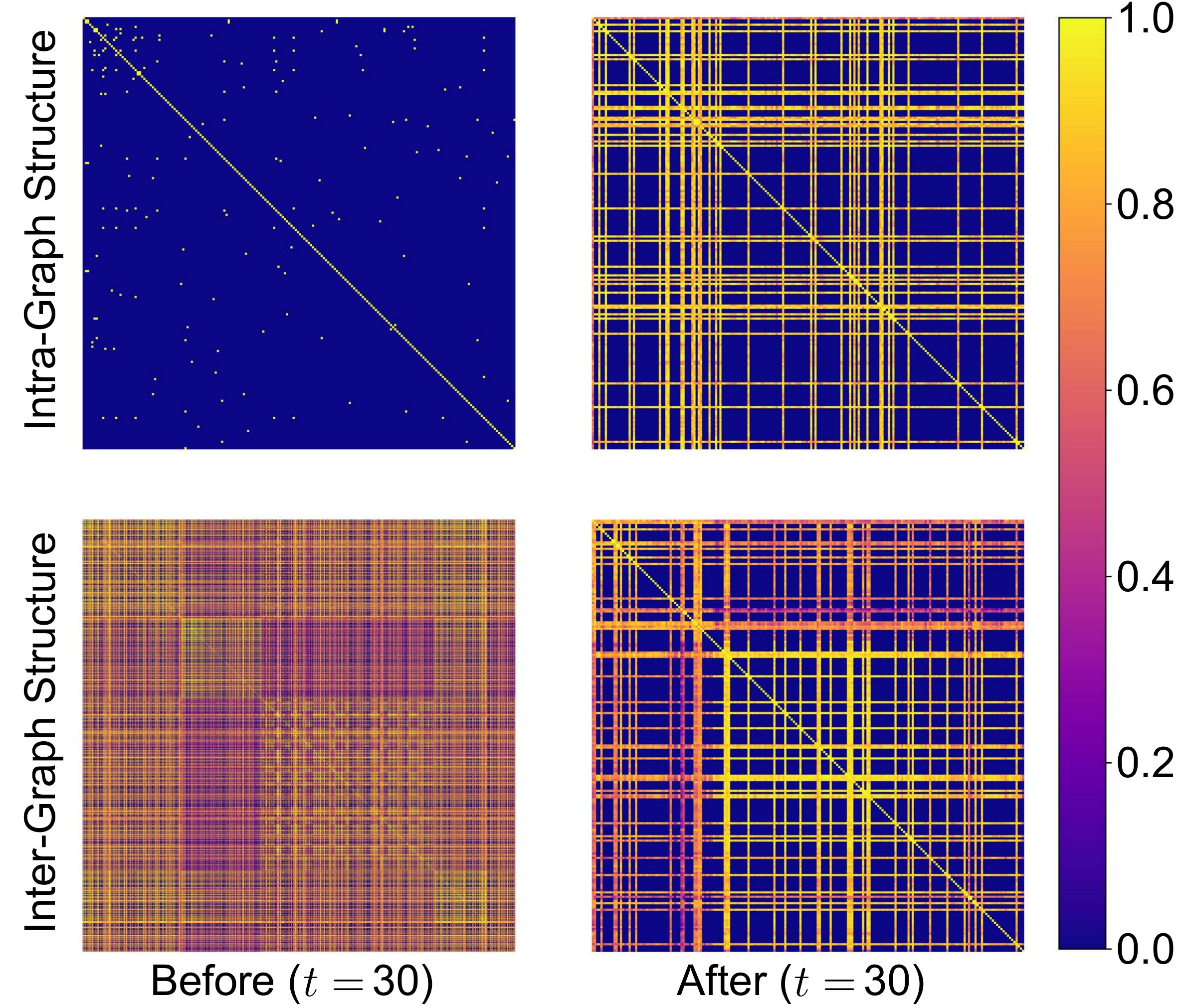}
        \vspace{-0.5em}
        \caption{Learned structure visualization.}
        \label{fig:vis}
    \end{minipage}
\vspace{-0.4em}
\end{figure*}

\subsection{Scaling Efficiency Analysis}
\label{sec:scale}
To evaluate efficiency of \modelname, we generate synthetic datasets by manipulating the node scale and sequence length based on the datasets introduced in Section~\ref{sec:data}. We plot the training time per epoch and GPU usage at peak time on Yelp in Figure~\ref{fig:effi_node_yelp} and Figure~\ref{fig:effi_length_yelp}, where the res cross indicates OOM (Out-Of-Memory). Additional results in Appendix~\ref{sec:scale_add}.

\subsubsection{Analysis.}
The results highlight superior efficiency of \modelname~in both the node scale and sequence length scaling scenarios, where its training time and GPU usage scale up near linearly, which is consistent with the theoretical analysis conclusion. For node scaling, compared with RDGSL, \modelname~reduced training time and GPU usage up to 71.3\% and 45.8\%, respectively. For sequence length scaling, compared with DGIB-Cat, \modelname~reduced training time and GPU usage up to 38.4\%, and 28.9\%, respectively. Beyond 6.5 times scaling, while RDGSL and DGIB both fail stuking by OOM, \modelname~still manages to operate within GPU limits even at the highest scaling factor tested, demonstrating its efficiency in long-span dynamic graphs.

\subsection{Ablation Study}
\label{sec:exp_abla}
We analyze the effectiveness of the three variants:
\begin{itemize}[leftmargin=*]
    \item \textbf{\modelname~(\textit{w/o KMP}):} We replace the efficient spatio-temporal kernelized message-passing in Section~\ref{sec:kernel} with the vanilla attention-based message-passing mechanism.
    \item \textbf{\modelname~(\textit{w/o SM}):} We remove the long-range dependencies selective modeling proposed in Section~\ref{sec:ssm}.
    \item \textbf{\modelname~(\textit{w/o PRI}):} We remove Principle of Relevant Information for DGSL regularizing term in Section~\ref{sec:pri}.
\end{itemize}
Results on clean datasets are shown in Figure~\ref{fig:abla_main}. Results for evasion and poisoning datasets are shown in Figure~\ref{fig:abla_add}.

\subsubsection{Analysis.}
Overall, the \modelname~outperforms the other two variants, \ie~``\textit{w/o SM}'' and ``\textit{w/o PRI}'', which validates the indispensable effectiveness of the long-range dependencies selective modeling mechanism and the PRI for DGSL regularize. We claim the exceeding performance of \modelname~(\textit{w/o KMP}) is within our expectation as the kernelized message-passing sacrifices effectiveness for efficiency. 

\subsection{Hyperparameter Sensitivity Analysis}
\label{sec:exp_hyper}
We perform evaluations on the sensitvity of important hyperparameters on COLLAB in Figure~\ref{fig:hyper_main}, where $\beta_1$ and $\beta_2$ control the importance of the distortion in PRI, $\mu$ and $\lambda$ play the trade-off role between node embeddings and loss terms. Results demonstrate the performance is sensitive to different values of the hyperparameters and contains a reasonable range. Further analysis is provided in Appendix~\ref{sec:hyper_more}.

\subsection{Visualization of Learned Dynamic Structures}
\label{sec:exp_vis}
We visualize edge weights for intra- and inter-graph structures of ACT before and after training in Figure~\ref{fig:vis}. Results demonstrate \modelname~can effectively emphasize on key structure patterns for prediction as well as denoising irrelevant features, which contributes to improving robustness.

\section{Conclusion}
In this paper, we present a robust and efficient DGSL framework named \modelname~with linear time complexity.
The kernelized message-passing behaves efficiently with state-discretized SSM. Learned structures are regularized with the proposed PRI for DGSL. Long-range dependencies and underlying predictive patterns are uncovered to strengthen robustness.
Extensive experiments demonstrate its superiority over 12 state-of-the-art baselines against adversarial attacks.

\section*{Acknowledgments}
The corresponding author is Jianxin Li. Authors of this paper are supported by the National Natural Science Foundation of China through grants No.623B2010, No.62225202, and No.62302023, and the Fundamental Research Funds for the Central Universities. We extend our sincere thanks to all authors for their valuable efforts and contributions.

\bibliography{aaai25}

\clearpage

\newpage
\appendix
\counterwithin{table}{section}
\counterwithin{figure}{section}
\counterwithin{equation}{section}

\section{Algorithms and Complexity Analysis}
\label{sec:alg}
The overall training process of \modelname~is illustrated in Algorithm~\ref{alg:alg_overall}, where Algorithm~\ref{alg:alg_overall} calls $\operatorname{DGSS}$ in Algorithm~\ref{alg:alg_selective}.
\vspace{-0.5cm}
\begin{algorithm}[h]
    \caption{Overall training pipeline of \modelname.}
    \label{alg:alg_overall}
    \textbf{Input}: Dynamic graph $\mathcal{G}^{1:T}$; Node features $\textbf{X}^{1:T+1}$; Labels\\\textcolor{white}{\textbf{Input}:} $\mathbf{Y}^{1:T}$ of the link occurrence; Model $f_{{\boldsymbol{\theta}}} = w\circ g$.\\
    \textbf{Parameter}: Number of training epochs $E$; Number of layers\\\textcolor{white}{\textbf{Parameter}:} $L$; Hyperparameters $\beta_1$, $\beta_1$, $\gamma$, $\lambda$ and $\mu$.\\
    \textbf{Output}: Refined dynamic graph $\hat{\mathcal{G}}^{1:T}$; Optimized parameter\\\textcolor{white}{\textbf{Output}:} $f_{{\boldsymbol{\theta}}}^\star = w\circ g$; Predicted label $\hat{\mathbf{Y}}^{T+1}$ of the link occur-\\\textcolor{white}{\textbf{Output}:} rence at time step $T+1$.
    
    \begin{algorithmic}[1]
        \STATE Initialize weights and model parameters $\boldsymbol{\theta}$ randomly;
        \STATE Initialize: $\mathbf{Z}^{1:T+1(0)} \gets \mathbf{X}^{1:T+1}$;
        \STATE Relative time encoding: $\mathbf{Z}^{1:T+1(0)}\!\gets\!\!\operatorname{RTE}(\mathbf{Z}^{1:T+1(0)})$;
        \FOR{epochs $e=1, 2, \cdots, E$}
        \FOR{layer $l=1, 2, \cdots, L$}
        \STATE Kernelized message-passing: $\mathbf{Z}_\text{MP}\gets$ Eq.~\eqref{eq:update_3};
        \STATE Average pooling: $\overline{\mathbf{Z}}_\text{MP}\gets\operatorname{AvgPool}(\mathbf{Z}_\text{MP})$;
        \STATE Structure query: $\hat{\mathbf{A}}_\text{intra}^{1:T}, \hat{\mathbf{A}}_\text{inter}^{1:T}\gets$ Eq.~\eqref{eq:intra_structure}, \eqref{eq:inter_structure};
        \STATE Long-range dependencies selective modeling:\\ $\mathbf{Z}_\text{Seq}\gets\operatorname{DGSS}(\overline{\mathbf{Z}}_\text{MP}, \hat{\mathbf{A}}_\text{inter}^{1:T})$ (see Algorithm~\ref{alg:alg_selective});
        \STATE Node feature update: $\hat{\mathbf{Z}}\gets$ Eq.~\eqref{eq:merge};
        \ENDFOR
        \STATE Predict: $\hat{\mathbf{Y}}^{T+1} = g(\hat{\mathbf{Z}}^{T+1(L)})$;
        \STATE $\mathcal{L}_\text{LP}\gets\operatorname{CE}(\mathbf{Y}^{T+1},\hat{\mathbf{Y}}^{T+1})$; $\mathcal{L}_\text{PRI}(\hat{\mathcal{G}}^{1:T})\gets$Eq.~\eqref{eq:structure_pri};
        \STATE Calculate the overall loss: $\mathcal{L}\!\gets\!\mathcal{L}_{\text{LP}} + \mu\cdot\mathcal{L}_\text{PRI}(\hat{\mathcal{G}}^{1:T})$;
        \STATE Update $\boldsymbol{\theta}$ by minimizing $\mathcal{L}$ and back-propagation.
        \ENDFOR
    \end{algorithmic}
\end{algorithm}

\vspace{-0.5cm}
\begin{algorithm}[!h]
    \caption{Dynamic Graph Selective Scan ($\operatorname{DGSS}$).}
    \label{alg:alg_selective}
    \textbf{Input}: The average-pooled representation $\overline{\mathbf{Z}}_\text{MP}\!\in\!\mathbb{R}^{B\times T\times N}$;\\
    \textcolor{white}{\textbf{Input}:} The inter-graph structures $\hat{\mathbf{A}}_\text{inter}^{1:T} \in \mathbb{R}^{B\times T\times N \times N}$.\\
    \textbf{Parameter}: The batch size $B$; The dynamic graph length $T$;\\\textcolor{white}{\textbf{Parameter}:} The number of latent states $N$; The hidden fea-\\\textcolor{white}{\textbf{Parameter}:} ture dimension $D$; The weight matrix $\mathbf{W}$.\\
    \textbf{Output}: The temporal semantics and long-range dependen-\\ \textcolor{white}{\textbf{Output}:} cies emerged $\mathbf{Z}_\text{Seq}\in\mathbb{R}^{B\times T\times N}$.
    
    \begin{algorithmic}[1]
        \STATE Randomly initialize $\boldsymbol{\mathcal{A}} \in \mathbb{R}^{N\times D}$;
        \STATE $\boldsymbol{\mathcal{B}}, \boldsymbol{\mathcal{C}} \in \mathbb{R}^{B\times T \times D} \gets \operatorname{Linear}_{D}(\overline{\mathbf{Z}}_\text{MP})$;
        \STATE $\boldsymbol{\Delta} \in \mathbb{R}^{B\times T \times N} \gets \operatorname{softplus}(\operatorname{Linear}_{N}(\overline{\mathbf{Z}}_\text{MP}))$;
        \STATE $\overline{\boldsymbol{\Delta}} \in \mathbb{R}^{B\times T \times N} \gets \operatorname{unsqueeze}_N(\boldsymbol{\Delta})\cdot \mathbf{W} \hat{\mathbf{A}}^{1:T}_\text{inter}$;
        \STATE $\overline{\boldsymbol{\mathcal{A}}} \in \mathbb{R}^{B\times T\times N\times D} \gets \exp(\overline{\boldsymbol{\Delta}}\boldsymbol{\mathcal{A}})$;
        \STATE $\overline{\boldsymbol{\mathcal{B}}} \in \mathbb{R}^{B\times T\times N\times D} \gets (\overline{\boldsymbol{\Delta}} \boldsymbol{\mathcal{A}})^{-1}(\exp (\overline{\boldsymbol{\Delta}}\boldsymbol{\mathcal{A}})-\boldsymbol{I})\overline{\boldsymbol{\Delta}}\boldsymbol{\mathcal{B}}$;
        \STATE Initialize latent state $\mathbf{H}\in\mathbb{R}^{B\times T\times N\times D}$ as zeros;
        \STATE Initialize the output $\mathbf{Z}_\text{Seq}\in\mathbb{R}^{B\times T\times N}$ as empty set;
        \FOR{$t$ {from} $1$ {to} $T-1$}
        \STATE \quad$\mathbf{H}_t = \overline{\boldsymbol{\mathcal{A}}}_t \mathbf{H}_{t-1} + \overline{\boldsymbol{\mathcal{B}}} (\overline{\mathbf{Z}}_\text{MP})_t$;
        \STATE \quad$(\mathbf{Z}_\text{Seq})_t = \boldsymbol{\mathcal{C}}_t\mathbf{H}_t$;
        \ENDFOR
        \STATE \textbf{return} $\mathbf{Z}_\text{Seq}\in\mathbb{R}^{B\times T\times N}$.
    \end{algorithmic}
\end{algorithm}

\noindent\textbf{Computational Complexity Analysis.} For brevity, denote $|\mathcal{V}|$ as the average number of nodes, $|\mathcal{E}|$ as the average number of edges, and $T$ denotes the graph length. Let $d$ be the dimension of the input node features, and $D_0$ be the dimension of the learned node embeddings after kernelized spatio-temporal message-passing, and $D$ be the last dimension of the latent state $\mathbf{H}$. We analyze the computational complexity of each part during training in \modelname~as follows.
\begin{itemize}[leftmargin=*]
    \item Linear input feature projection layer: $\mathcal{O}(T|\mathcal{V}|dD_0)$.
    \item Relative time encoding ($\operatorname{RTE}(\cdot)$) layer: $\mathcal{O}(T|\mathcal{V}|D_0^2)$.
    \item Kernelized message-passing layer: $\mathcal{O}(2LT(|\mathcal{V}|+|\mathcal{E}|)D_0)$, where $L$ denotes the number of convolution layers.
    \item Long-range dependencies selective modeling: $\mathcal{O}(TD)$.
    \item Robust relevant structure regularizing: $\mathcal{O}(2T(|\mathcal{V}|+|\mathcal{E}|))$.
    \item Link predictor (MLP), feature aggregations, and activation functions: constant complexity brought about by addition operations with respect to feature dimensions (ignored).
\end{itemize}
The overall computational complexity of \modelname~is,
\begin{align*}
        \mathcal{O}(T|\mathcal{V}|dD_0)+\mathcal{O}(T|\mathcal{V}|D_0^2)+ \mathcal{O}(2LT(|\mathcal{V}|+|\mathcal{E}|)D_0)\notag\\+\mathcal{O}(TD)+\mathcal{O}(2T(|\mathcal{V}|+|\mathcal{E}|)).
\end{align*}
As the dimensions $d$, $D$, $D_0$, and the number of convolution layers $L$ are relatively small compared to the average number of nodes, edges, and graph length, the overall computational complexity can then be approximately reduced to,
\begin{align}
     &\mathcal{O}(T|\mathcal{V}|)+\mathcal{O}(T(|\mathcal{V}|+|\mathcal{E}|))+\mathcal{O}(T) \\ =\; &\mathcal{O}(T(|\mathcal{V}|+|\mathcal{E}|)).
\end{align}
To summarize, the overall computational complexity of the proposed \modelname~is \textbf{linear} concerning the average number of nodes, edges, and graph length, making it significantly more efficient than state-of-the-art DGNNs, especially when the original graphs are less dense. We also conducted empirical studies on training time, testing time, and peak memory consumption when scaling the dynamic graphs up to 8 times the original inputs. The results, provided in Section~\ref{sec:scale} and Appendix~\ref{sec:scale_add}, further demonstrate the superior efficiency of \modelname~compared to the competitive baselines. Additionally, based on our experimental experience, \modelname~can be trained and tested with the hardware configurations (including memory requirements) listed in Appendix~\ref{sec:config}.
\section{Proofs}
\label{sec:proof}
In this section, we provide proofs for the approximation error bounds of the Gumbel-Softmax kernel in Section~\ref{sec:kernel} and loss equivalence for the edge-level constraints in Section~\ref{sec:pri}.

\subsection{Proof of the Gumbel-Softmax Kernel Approximation Error Bound}
\label{sec:proof_1}
Section~\ref{sec:kernel} proposed the kernelized message-passing mechanism to update node embeddings and learn spatio-temporal edge weights efficiently, \ie,
\begin{equation}
    \mathbf{z}_u^{t(l+1)} = \frac{\phi(\mathbf{W}\mathbf{z}_u^{t(l)})^\top\sum_v\phi(\mathbf{W}\mathbf{z}_v^{\cdot(l)})\mathbf{W}\mathbf{z}_v^{\cdot(l)\top}}{\phi(\mathbf{W}\mathbf{z}_u^{t(l)})^\top\sum_m\phi(\mathbf{W}\mathbf{z}_m^{\cdot(l)})},
\end{equation}
where $\phi(\cdot)$ is the approximation for Softmax kernel $k(\cdot,\cdot)$ and is implemented by the Positive Random Features (PRF).
\subsubsection{Gumbel-Softmax Kernel Approximation.}
To enable differentiable optimization over discrete dynamic graphs, we first incorporate Eq.~\eqref{eq:update_2} with the categorical parameterization trick with the Gumbel-Softmax~\cite{jang2022categorical,wu2022nodeformer}, \ie~(we omit the layer superscripts, the activation function $\sigma$, and the limits of summations always for nodes in the $\mathcal{N}(u)^{t-1:t}$ for brevity),
\begin{align}
    \mathbf{z}_u^{t} &= \sum_v \frac{\exp(((\mathbf{W}\mathbf{z}_u^{t})^\top(\mathbf{W}\mathbf{z}_v^{\cdot})+g_v)/\tau)\mathbf{W}\mathbf{z}_v^{\cdot}}{\sum_m \exp(((\mathbf{W}\mathbf{z}_u^{t})^\top(\mathbf{W}\mathbf{z}_m^{\cdot})+g_m)/\tau)}\notag\\
    &= \sum_v \frac{k(\mathbf{W}\mathbf{z}_u^{t}/\sqrt{\tau}, \mathbf{W}\mathbf{z}_v^{\cdot}/\sqrt{\tau})\exp(g_v/\tau)\mathbf{W}\mathbf{z}_v^{\cdot}}{\sum_m k(\mathbf{W}\mathbf{z}_u^{t}/\sqrt{\tau}, \mathbf{W}\mathbf{z}_m^{\cdot}/\sqrt{\tau})\exp(g_m/\tau})\notag\\
    &= \sum_v \frac{\phi(\mathbf{W}\mathbf{z}_u^{t}/\sqrt{\tau})^\top\phi(\mathbf{W}\mathbf{z}_v^{\cdot}/\sqrt{\tau})\exp(g_v/\tau)\mathbf{W}\mathbf{z}_v^{\cdot}}{\sum_m \phi(\mathbf{W}\mathbf{z}_u^{t}/\sqrt{\tau})^\top \phi(\mathbf{W}\mathbf{z}_m^{\cdot}/\sqrt{\tau})\exp(g_m/\tau)}\notag\\
    &= \frac{\phi(\mathbf{W}\mathbf{z}_u^{t}/\sqrt{\tau})^\top\sum_v\exp(g_v/\tau)\phi(\mathbf{W}\mathbf{z}_v^{\cdot}/\sqrt{\tau})\mathbf{W}\mathbf{z}_v^{\cdot\top}}{\phi(\mathbf{W}\mathbf{z}_u^{t}/\sqrt{\tau})^\top \sum_m\exp(g_m/\tau)\phi(\mathbf{W}\mathbf{z}_m^{\cdot}/\sqrt{\tau})},
    \label{eq:gumbel}
\end{align}
where the element $g$ is independently and identically sampled from the Gumbel distribution $G(0,1)$, which first samples $u$ from the uniform distribution $U(0,1)$, and then obtain $g=-\log(-\log(u))$. $\tau$ is a temperature coefficient.

Eq.~\eqref{eq:gumbel} provides the continuous relaxation for sampling neighbor nodes of node $u$ from the categorical distribution $\operatorname{Cat}(\boldsymbol{\pi}_u)$, where $\boldsymbol{\pi}_u$ is determined by the edge weights $\hat{\alpha}_{uv}$. This guarantees a differentiable optimization over the discrete inter- and intra-graph structures. The temperature parameter $\tau$ regulates the proximity to hard discrete samples: smaller $\tau$ results in a distribution closer to a one-hot categorical distribution, while larger $\tau$ yields a distribution closer to a uniform distribution across each category. Similar to Eq.~\eqref{eq:update_3}, the two summations in Eq.~\eqref{eq:gumbel} greatly contribute to decreasing the quadratic complexity as they can be computed once and stored for each $u$. Also, kernel $k(\cdot,\cdot)$ can be estimated by the same Positive Random Features (PRF) for the Gumbel-Softmax approximation.

\subsubsection{Approximation Error Bound.}
We next present the error bound of approximating the Gumbel-Softmax kernel $k(\cdot,\cdot)$ with the Gumbel Softmax Positive Random Features $\phi(\cdot)$.
\begin{Prop}[Gumbel-Softmax Kernel Approximation Error Bound]
\label{prop:kernel_bound}
Suppose $\|\mathbf{W}\mathbf{z}_u^t\|_2$ and $\|\mathbf{W}\mathbf{z}_v^t\|_2$ are bounded by some $r$, then for any positive $\epsilon$, the approximation error $\Delta = |\phi(\mathbf{W}\mathbf{z}_u^{t}/\sqrt{\tau})^\top\phi(\mathbf{W}\mathbf{z}_v^{t}/\sqrt{\tau}) - k(\mathbf{W}\mathbf{z}_u^{t}/\sqrt{\tau}, \mathbf{W}\mathbf{z}_v^{t}/\sqrt{\tau})|$ satisfies,
\begin{equation}
    \mathbb{P}(\Delta < \epsilon) \ge 1-\frac{\exp(\frac{6r}{\tau})}{m\epsilon^2},
\end{equation}
where $m$ denotes the projection dimension in the Reproducing Kernel Hilbert Space $\mathcal{H}$.
\end{Prop}
\begin{proof}
    We first introduce the following lemma introduced in \citet{choromanski2020rethinking} and \cite{wu2022nodeformer}.
    \begin{Lemma}[\citet{choromanski2020rethinking, wu2022nodeformer}]
    \label{lemma:mean_and_var}
    Given the Positive Random Feature defined in Eq.~\eqref{eq:prf}, where its Softmax kernel approximation is,
    \begin{align}
        &\exp(\mathbf{x},\mathbf{y})\doteq \phi(\mathbf{x})^\top\phi(\mathbf{y})\notag\\
        &=\frac{1}{m}\sum_{i=1}^m\left[ \exp\left(\boldsymbol{\omega}_i^\top\mathbf{x}-\frac{\|\mathbf{x}\|_2^2}{2}\right)\exp\left(\boldsymbol{\omega}_i^\top\mathbf{y}-\frac{\|\mathbf{y}\|_2^2}{2}\right) \right],
    \end{align}
    has the quantified mean and variance, \ie,
    \begin{align}
        &\mathbb{E}(\phi(\mathbf{x})^\top\phi(\mathbf{y})) = \exp(\mathbf{x}^\top\mathbf{y}),\label{eq:mean}\\
        &\mathbb{V}(\phi(\mathbf{x})^\top\phi(\mathbf{y})) = \frac{1}{m} \exp(\|\mathbf{x}+\mathbf{y}\|_2^2)\exp^2(\mathbf{x}^\top\mathbf{y})\notag\\
        &\color{white}{\mathbb{V}(\phi(\mathbf{x})^\top\phi(\mathbf{y})) = \;}\color{black}(1-\exp(-\|\mathbf{x}+\mathbf{y}\|_2^2)).\label{eq:var}
    \end{align}
    \end{Lemma}
    Lemma~\ref{lemma:mean_and_var} demonstrates that the Positive Random Feature can achieve unbiased and deterministic approximation for the Softmax kernel. We next utilize Lemma~\ref{lemma:mean_and_var} to derive the probability inequality with Chebyshv's Inequality~\cite{feller1991introduction} that let $\mathbf{X}$ be a random variable with finite mean $\mu$ and non-zero variance $\sigma^2$, then for any positive $\epsilon$,
    \begin{equation}
        \mathbb{P}(|\mathbf{X}-\mu|<\epsilon)\ge 1-\frac{\sigma^2}{\epsilon^2}.
    \label{eq:cheb}
    \end{equation}
    We integrate Eq.~\eqref{eq:cheb} with Eq.~\eqref{eq:mean} and Eq.~\eqref{eq:var}, \ie,
    \begin{align}
        &\mathbb{P}(|\phi(\mathbf{x})^\top\phi(\mathbf{y})-\exp(\mathbf{x}^\top\mathbf{y})|<\epsilon)\notag\\&\ge 1-\frac{\exp(\|\mathbf{x}+\mathbf{y}\|_2^2)\exp^2(\mathbf{x}^\top\mathbf{y})(1-\exp(-\|\mathbf{x}+\mathbf{y}\|_2^2))}{m\epsilon^2}\notag\\
        &= 1-\frac{\exp(\|\mathbf{x}+\mathbf{y}\|_2^2)\exp^2(\mathbf{x}^\top\mathbf{y})}{m\epsilon^2}.
    \end{align}
    Further, we replace $\mathbf{x}$ as $\mathbf{W}\mathbf{z}_u^{t}/\sqrt{\tau}$, and $\mathbf{y}$ as $\mathbf{W}\mathbf{z}_v^{t}/\sqrt{\tau}$, we have,
    \begin{align}
        \mathbb{P}(\Delta < \epsilon) \ge 1- \frac{\exp(\|\frac{\mathbf{W}\mathbf{z}_u^{t}+\mathbf{W}\mathbf{z}_v^{t}}{\sqrt{\tau}}\|_2^2+2\frac{(\mathbf{W}\mathbf{z}_u^{t})^\top\mathbf{W}\mathbf{z}_v^{t}}{\tau})}{m\epsilon^2}.
    \end{align}
    Since $\|\mathbf{W}\mathbf{z}_u^t\|_2$ and $\|\mathbf{W}\mathbf{z}_v^t\|_2$ are bounded by $r$, such that,
    \begin{equation}
        \left\|\frac{\mathbf{W}\mathbf{z}_u^{t}+\mathbf{W}\mathbf{z}_v^{t}}{\sqrt{\tau}}\right\|_2^2 \le \frac{4r}{\tau},\; 2\frac{(\mathbf{W}\mathbf{z}_u^{t})^\top\mathbf{W}\mathbf{z}_v^{t}}{\tau} \le \frac{2r}{\tau}.
    \end{equation}
   Then we achieve,
    \begin{equation}
        \mathbb{P}(\Delta < \epsilon) \ge 1-\frac{\exp(\frac{6r}{\tau})}{m\epsilon^2},
    \label{eq:bound_2}
    \end{equation}
    which indicates that the approximation error is upper bound by a quantified probability. We conclude the proof.
\end{proof}
\subsubsection{Empirical Analysis.}
$\epsilon$ is any positive real number, which means the inequality holds true when $\epsilon\to$ 0. In practice, we prefer setting $\tau\to$ 0 as smaller $\tau$ results in a distribution closer to the one-hot categorical distribution. $m$ is some constants that can be ignored. Considering that the growth rate of the exponential function is much faster than that of the quadratic function, as $\epsilon$ and $\tau$ simultaneously approach 0, the entire fraction of the RHS in Eq.~\eqref{eq:bound_2} approaches 0, which means the probability $\mathbb{P}(\Delta < \epsilon)$ approaches 1. In other words, the probability that the Gumbel-Softmax kernel's approximation error by Positive Random Features approaches zero is exceptionally high, further validating its effectiveness and superiority.

\subsubsection{Efficient Structure Query.}
As Eq.~\eqref{eq:update_3} combines both the edge reweighting and message-passing process in a unified and implicit manner, we can still explicitly obtain the optimized edge weights $\hat{\alpha}_{uv}^{t-1:t}$ by efficiently querying the approximated kernels of any node pairs~\cite{wu2022nodeformer}. For node $u$ in $\mathcal{V}^t$, nodes $v$ and $m$ in $\mathcal{N}(u)^{t-1:t}$,
\begin{align}
    \hat{\alpha}_{uv}^{t-1:t(l)}&=\frac{\exp(\sigma((\mathbf{W}\mathbf{z}_u^{t(l)})^\top(\mathbf{W}\mathbf{z}_v^{\cdot(l)})))}{\sum_m \exp(\sigma((\mathbf{W}\mathbf{z}_u^{t(l)})^\top(\mathbf{W}\mathbf{z}_m^{\cdot(l)})))}\notag\\
    &\doteq \frac{\phi(\mathbf{W}\mathbf{z}_u^{t(l)})^\top\phi(\mathbf{W}\mathbf{z}_v^{\cdot(l)})}{\phi(\mathbf{W}\mathbf{z}_u^{t(l)})^\top\sum_m\phi(\mathbf{W}\mathbf{z}_m^{t(l)})}.
\end{align}
Similar to the approximation in Eq~.\ref{eq:update_3}, the two summations can be computed once and stored for each $u$. For intra-graph structure query, we only need to query the existing edges of the input dynamic graphs, which consumes $\mathcal{O}(T(|\mathcal{V}|+|\mathcal{E}|))$ time complexity (for Eq.~\eqref{eq:intra_structure} and Eq.~\eqref{eq:L_edge}). For inter-graph structure, as there does not exist original inter-graph structures in the input graphs, we initialize the cross-graph structures by calculating the cosine similarity between graphs with the initialized node embeddings before training and select the top-$k$ edges where $k$ equals $|\mathcal{E}|$. Such that, the inter-graph structure query also consumes $\mathcal{O}(T(|\mathcal{V}|+|\mathcal{E}|))$ time complexity. The overall computational complexity for the structure query is $\mathcal{O}(2T(|\mathcal{V}|+|\mathcal{E}|))$without increasing the total computational burden.

\subsection{Proof of Loss Equivalence for Edge-Level Constraints}
\label{sec:proof_2}
In this section, we provide proof of the loss equivalence for the edge-level constraints. Particularly, we approximately decompose the overall PRI for DGSL objective (Eq.~\eqref{eq:pri}) into respective spatial-wise and temporal-wise regularizing with the learned intra- and inter-graph structures (Eq.~\eqref{eq:structure_pri}). Additionally, to regularize intra-graph structures, we transform the divergence term $\mathcal{D}(\hat{\mathbf{A}}^{1:T}_\text{intra}\|\mathbf{A}^{1:T}_\text{intra})$ into the edge-level constraints $\mathcal{L}_\text{edge}$ (Eq.~\eqref{eq:L_intra}).

\subsubsection{Loss Equivalence.}
Next, we introduce the proof of the loss equivalence for edge-level constraints transformation.
\begin{Prop}[Edge-Level Constraints Loss Equivalance]
    Given the learned intra-graph structures $\hat{\mathbf{A}^{1:T}_\text{intra}}$ and the original structures $\mathbf{A}^{1:T}_\text{intra}$, the divergence loss $\mathcal{D}(\hat{\mathbf{A}}^{1:T}_\text{intra}\|\mathbf{A}^{1:T}_\text{intra})$ is equivalent to the edge-level constraints $\mathcal{L}_\text{edge}$.
\end{Prop}

\begin{proof}
    For brevity, we denote the original graph structures as $\mathbf{A}^{1:T}$, and the learned structures as $\hat{\mathbf{A}}^{1:T}$. Considering the structures of a single graph snapshot, the target is converted to prove $\mathcal{D}(\hat{\mathbf{A}}^t\|\mathbf{A}^{t})$ is equivalent to edge-level constraints,
    \begin{equation}
        \mathcal{L}_\text{edge}^t = -\frac{1}{N}\sum_{u,v\in\mathcal{E}^t} \frac{1}{d(u)}\log \hat{\alpha}^t_{uv}.
    \end{equation}
    We first define the probability of the two adjacency matrices,
    \begin{equation}
        \mathbb{P}_{uv} = \frac{\hat{\alpha}^{t}_{uv}}{\sum_{u,v}\hat{\alpha}^{t}_{uv}},\quad
        \mathbb{Q}_{uv} = \frac{{\alpha}^{t}_{uv}}{\sum_{u,v}{\alpha}^{t}_{uv}}.
    \end{equation}
    Next, the KL-divergence is implemented to the divergence,
    \begin{equation}
        \color{white}{xxxx}\color{black}\mathcal{D}_\text{KL}(\mathbb{P}\|\mathbb{Q}) = \sum_{u,v}\mathbb{P}_{uv}\log\frac{\mathbb{P}_{uv}}{\mathbb{Q}_{uv}}.\;\;\color{white}{xx}\color{black}
    \end{equation}
    The optimization goal is to minimize the KL-divergence between $\hat{\alpha}^t$ and $\alpha^{t}$, \ie, we expect there exists a limited deviation between the learned structures and the original ones,
    \begin{equation}
        \mathcal{D}_\text{KL}(\mathbb{P}\|\mathbb{Q}) = \sum_{u,v} \frac{\hat{\alpha}^{t}_{uv}}{\sum_{u,v}\hat{\alpha}^{t}_{uv}} \log \frac{\frac{\hat{\alpha}^{t}_{uv}}{\sum_{u,v}\hat{\alpha}^{t}_{uv}}}{\frac{{\alpha}^{t}_{uv}}{\sum_{u,v}{\alpha}^{t}_{uv}}},
    \end{equation}
    which can be simplified as,
    \allowdisplaybreaks
    \begin{align}
        &\mathcal{D}_\text{KL}(\mathbb{P}\|\mathbb{Q}) = \sum_{u,v}\frac{\hat{\alpha}^{t}_{uv}}{\sum_{k,l}\hat{\alpha}^{t}_{kl}}\log \frac{\hat{\alpha}^{t}_{uv}\sum_{k,l}\alpha^{t}_{kl}}{\alpha^{t}_{uv}\sum_{k,l}\hat{\alpha}^{t}_{kl}}\notag\\
        &=\sum_{u,v}\frac{\hat{\alpha}^{t}_{uv}}{\sum_{k,l}\hat{\alpha}^{t}_{kl}}\Big( \log\hat{\alpha}^{t}_{uv}+\log\sum_{k,l}\alpha^{t}_{kl}\notag-\log\alpha^{t}_{uv}\notag\\&\quad\quad\quad\quad\quad\quad\quad\quad\quad\quad\quad\quad\quad\quad\quad-\log\sum_{k,l}\hat{\alpha}^{t}_{kl} \Big)\notag\\
        &=\sum_{u,v}\frac{\hat{\alpha}^{t}_{uv}}{\sum_{k,l}\hat{\alpha}^{t}_{kl}}\log\hat{\alpha}^{t}_{uv}+\log\sum_{k,l}\alpha^{t}_{kl}\sum_{u,v}\frac{\hat{\alpha}^{t}_{uv}}{\sum_{k,l}\hat{\alpha}^{t}_{kl}}\notag\\
        &\quad-\sum_{u,v}\frac{\hat{\alpha}^{t}_{uv}}{\sum_{k,l}\hat{\alpha}^{t}_{kl}}\log\alpha^{t}_{uv}-\log\sum_{k,l}\hat{\alpha}^{t}_{kl}\sum_{u,v}\frac{\hat{\alpha}^{t}_{uv}}{\sum_{k,l}\hat{\alpha}^{t}_{kl}},\label{eq:kl_temp}
    \end{align}
    where, 
    \begin{equation}
        \log\sum_{k,l}\alpha^{t}_{kl}=\log\sum_{k,l}\hat{\alpha}^{t}_{kl}=0,
    \end{equation}
    Such that, Eq.~\eqref{eq:kl_temp} can be further simplified to,
    \begin{align}
        \mathcal{D}_\text{KL}(\mathbb{P}\|\mathbb{Q}) &= \sum_{u,v}\frac{\hat{\alpha}^{t}_{uv}}{\sum_{k,l}\hat{\alpha}^{t}_{kl}}\left(\log\hat{\alpha}^{t}_{uv}-\log\alpha^{t}_{uv}\right)\notag\\
        &=\sum_{u,v}\frac{\hat{\alpha}^{t}_{uv}}{\sum_{k,l}\hat{\alpha}^{t}_{kl}} \log\frac{\hat{\alpha}^{t}_{uv}}{\alpha^{t}_{uv}}.
        \label{eq:kl-divergence}
    \end{align}
    Since the KL-divergence is non-negative, its minimum value of zero occurs when $\hat{\alpha}^{t}_{uv}/\alpha^{t}_{uv}$ in Eq.~\eqref{eq:kl-divergence} equals 1. Given that $\alpha^{t}_{uv}$ represents binary input edge weights (fixed at 1), and $\hat{\alpha}^{t}_{uv}$ denotes the learned weights ranging between 0 and 1, minimizing the KL-divergence is equivalent to maximizing $\hat{\alpha}^{t}_{uv}$. This objective shares the same optimization goal with that of $\mathcal{L}_\text{edge}^t$, which is designed to minimize $\mathcal{L}_\text{edge}^t$ by maximizing $\hat{\alpha}^{t}_{uv}$. By generalizing the conclusion from $\mathbf{A}^t$ to $\mathbf{A}^{1:T}$, we conclude the proof.
\end{proof}

\section{Experiment Details}
\label{sec:exp_detail}
In this section, we provide additional experiment details.
\subsection{Datasets Details}
We use three real-world datasets to evaluate the effectiveness of \modelname~on the challenging future link prediction task. 
\begin{itemize}[leftmargin=*]
    \item \textbf{COLLAB}\footnote{\url{https://www.aminer.cn/collaboration}}~\cite{tang2012cross}\textbf{:} This academic collaboration dataset encompasses publications from 1990 to 2006 and consists of 16 graph snapshots detailing dynamic citation networks among authors. Nodes and edges represent authors and their co-authorships, respectively. Each edge possesses attributes based on the co-authored publication's field, including ``Data Mining'', ``Database'', ``Medical Informatics'', ``Theory'', and ``Visualization''. We employed the word2vec~\cite{mikolov2013efficient} to generate 32-dimensional node features from paper abstracts.
    \item \textbf{Yelp}\footnote{\url{https://www.yelp.com/dataset}}~\cite{sankar2020dysat}\textbf{:} This dataset includes customer reviews on businesses from a crowd-sourced local business review and social networking platform, spanning from January 2019 to December 2020 in 24 graph snapshots. In this network, nodes correspond to customers or businesses, while edges are indicative of review activities. The edges are classified into five attributes based on business categories: ``Pizza'', ``American (New) Food'', ``Coffee \& Tea'', ``Sushi Bars'', and ``Fast Food''. We applied word2vec~\cite{mikolov2013efficient} to extract 32-dimensional node features from the reviews.
    \item \textbf{ACT}\footnote{\url{https://snap.stanford.edu/data/act-mooc.html}}~\cite{kumar2019predicting}\textbf{:} This dataset captures student interactions within a popular MOOC platform over 30 graph snapshots within a single month. Nodes represent students or action targets, and edges signify student actions. We utilized K-Means to categorize action features into five distinct groups. Subsequently, we mapped these categorized features onto each student or action target, enhancing the original 4-dimensional features to 32 dimensions using a linear transformation.
\end{itemize}

Statistics of the three datasets are concluded in Table~\ref{tab:datasets}. Each dataset exhibits varying time spans and degrees of temporal granularity spanning 16 years, 24 months, and 30 days, thereby encapsulating a broad spectrum of real-world scenarios. Among these, COLLAB poses the greatest challenge for the future link prediction task. Its complexity is attributed to the extended time span and the coarsest temporal granularity. In addition, the considerable variation in link properties significantly challenges the model's robustness.

\begin{table}[!t]
\centering
\resizebox{\linewidth}{!}{
    \begin{tabular}{lrrrrr}
    \toprule
    \textbf{Dataset} & \textbf{\# Node} & \textbf{\# Link}  & \textbf{\makecell[r]{\# Link\\Type}} & \textbf{\makecell[r]{Length\\(Split)}} & \textbf{\makecell[r]{Temporal\\Granularity}} \\
    \midrule
    COLLAB & 23,035 & 151,790  & 5 & 16 (10/1/5)   & year \\
    Yelp  & 13,095 & 65,375  & 5 & 24 (15/1/8) & month \\
    ACT   & 20,408 & 202,339  & 5 & 30 (20/2/8)  & day\\
    \bottomrule
    \end{tabular}%
}
\caption{Statistics of the dynamic graph datasets.}
\label{tab:datasets}
\end{table}

\subsection{Baseline Details}
In this section, we introduce the baselines incorporated in the experiments. Hyperparameters of all baselines follow the recommended values from their papers and are fine-tuned for fairness. 
\begin{itemize}[leftmargin=*]
    \item \textbf{Static GNNs:}
    For static GNNs, we adapt them to the dynamic scenarios by stacking GNN layers with the vanilla LSTMs~\cite{hochreiter1997long}.
    \begin{itemize}[label=$\circ$]
        \item \textbf{GAE}~\cite{kipf2016variational}\textbf{:} employs Graph Convolutional Networks (GCNs)~\cite{kipf2022semi} as the encoder to learn node embeddings and reconstructs adjacency matrices via inner product. It is widely used for unsupervised node clustering and link prediction.
        \item \textbf{VGAE}~\cite{kipf2016variational}\textbf{:} extending GAE by adding the variational framework for generating probabilistic node embeddings, enhancing modeling of uncertainties in graph-structured data.
        \item \textbf{GAT}~\cite{velivckovic2018graph}\textbf{:} adapts attention mechanisms to weigh the significance of neighboring nodes dynamically, improving performance in node and graph classification tasks, particularly for graphs with irregular structures.
    \end{itemize}
    \item \textbf{Dynamic GNNs (DGNNs):}
    \begin{itemize}[label=$\circ$]
        \item \textbf{GCRN}~\cite{seo2018structured}\textbf{:} combines GCNs with recurrent neural architectures to model dynamic graphs effectively. This approach leverages the spatial and temporal features inherent in graph data, making it ideal for time-series predictions on graph-structured data.
        \item \textbf{EvolveGCN}~\cite{pareja2020evolvegcn}\textbf{:} introduces an evolutionary mechanism to GCNs, allowing node representations to evolve over time. EvolveGCN utilizes an RNN-based approach to update GCN weights dynamically, adapting to changes in the graph structure and enhancing performance in dynamic scenarios.
        \item \textbf{DySAT}~\cite{sankar2020dysat}\textbf{:} employs the self-attention mechanism across both the spatial and temporal dimensions, learning node representations that capture dependencies over time. DySAT is particularly effective for scenarios where node interactions frequently change.
        \item \textbf{SpoT-Mamba}~\cite{choi2024spot}\footnote{It is noteworthy that SpoT-Mamba is initially proposed for sequence prediction and is not capable of handling dynamic graphs with multi-dimensional node features, which falls outside the scope of our research. However, as it incorporates the highly efficient and advantageous Mamba~\cite{gu2023mamba}, and represents the cutting-edge related work, we have included it in our baselines. To mitigate differences between the research scope, we have replaced the multi-dimensional initial features of nodes in our dataset with the results from SpoT-Mamba's Walk Sequence Embedding.}\textbf{:} leveraging the state space model Mamba~\cite{gu2023mamba} for its prowess in capturing long-range dependencies, SpoT-Mamba is a novel framework for spatio-temporal graph forecasting. It generates node embeddings through node-specific walk sequences and employs temporal scans to capture long-range spatio-temporal dependencies.
    \end{itemize}
    \item \textbf{Dynamic Graph Structure Learning (DGSL):}
    \begin{itemize}[label=$\circ$]
        \item \textbf{TGSL}~\cite{zhang2023time}\textbf{:} enhances Temporal Graph Networks (TGNs)~\cite{rossi2020temporal} by predicting and incorporating time-sensitive edges to address the challenges of noisy and incomplete temporal graphs. TGSL optimizes graph structures end-to-end for improved performance on temporal link prediction tasks.
        \item \textbf{RDGSL}~\cite{zhang2023rdgsl}\textbf{:} combats noise in dynamic graphs through the dynamic graph filter and the temporal embedding learner. RDGSL dynamically addresses noise, focusing on denoising and ignoring disruptive interactions, leading to substantial performance gains in classification tasks on evolving graphs.
    \end{itemize}
    \item \textbf{Robust (D)GNNs:}
    \begin{itemize}[label=$\circ$]
        \item \textbf{RGCN}~\cite{zhu2019robust}\textbf{:} enhances robustness of GCNs against potential adversarial attacks using Gaussian distributions as node representations and introducing a variance-based attention mechanism. RGCN helps absorb adversarial effects, improving node classification accuracy significantly on benchmark graphs compared to GCNs. We adapt RGCN to the dynamic scenarios by stacking GNN layers with the vanilla LSTMs~\cite{hochreiter1997long}.
        \item \textbf{WinGNN}~\cite{zhu2023wingnn}\textbf{:} integrates a meta-learning strategy with a random gradient aggregation mechanism to model dynamics in graph neural networks efficiently. By eliminating the need for additional temporal encoders and utilizing a randomized sliding-window for gradient updates, WinGNN reduces parameter size and improves robustness, demonstrating superior performance across multiple dynamic network datasets.
        \item \textbf{DGIB}~\cite{yuan2024dynamic}\textbf{:} applies the Information Bottleneck to enhance the robustness of dynamic graph neural networks. DGIB promotes minimal, sufficient, and consensual representations to combat adversarial attacks, optimizing information flow in dynamic settings and showing remarkable robustness in link prediction tasks.
    \end{itemize}
\end{itemize}

\subsection{Experiment Setting Details}
\label{sec:exp_settings_add}
In this section, we introduce detailed experiment settings.

\subsubsection{Detailed Settings for Section~\ref{sec:exp_non_targeted}.}
We construct synthetic datasets by independently applying random perturbations to graph structures and node features. For the structural attack, each of the three real-world datasets is divided into multiple partial dynamic graphs according to their edge properties. We randomly exclude edges with a specific attribute, and the remaining edges are sequentially split into training, validation, and testing sets. Notably, edges with unseen attributes are only available during the testing phase, making the task more realistic and challenging. Furthermore, all attribute-related features are removed before training. For the feature attack, random Gaussian noise is added to each dimension of the node features, scaled according to the reference amplitude of the original features.

\subsubsection{Detailed Settings for Section~\ref{sec:exp_targeted}.}
We evaluate \modelname~under two common adversarial attack scenarios: evasion attack and poisoning attack. In the evasion attack setting, \modelname~is trained on the original clean datasets, and its link prediction performance is then evaluated on the attacked test datasets. Here, the model parameters are updated only once. For the poisoning attack, \modelname~is trained directly on the attacked training datasets, with three separate perturbation levels, and performance is reported on the corresponding test datasets. The model parameters for each of these three attacks are independently optimized. The relative performance decrease is calculated by dividing the difference between the AUC on the poisoned dataset and that on the clean dataset, by the AUC on the clean dataset. The average relative performance decrease is calculated by the average of three relative performance decreases.

\subsubsection{Detailed Settings for Section~\ref{sec:scale}.}
We compare the efficiency of \modelname~with RDGSL~\cite{zhang2023rdgsl} and DGIB-Cat~\cite{yuan2024dynamic}. For the node scaling settings, we remove a certain number of nodes from the original clean datasets, along with their connected edges. We report the training time per epoch (s). Note that, the training time per epoch does not affect the total training time to convergence for we utilize the early stopping mechanism, which halts the training after 50 consecutive epochs without improvement on the validation set, even though the training was initially set to run for 1,000 epochs. For the sequence length scaling settings, we record the peak GPU memory usage (GB) during training. The hardware environment used for training is detailed in Appendix~\ref{sec:config}.

\subsubsection{Detailed Settings for Section~\ref{sec:exp_abla}.}
We make three variants for ablation studies. For \textbf{\modelname~(\textit{w/o KMP})}, we replace the spatio-temporal kernelized message-passing mechanism proposed in Section~\ref{sec:kernel} by attention-based message-passing networks, \ie, GATs~\cite{velivckovic2018graph} for spatial-wise aggregation, and LSTMs~\cite{hochreiter1997long} for temporal convolution. \textbf{For \modelname~(\textit{w/o SM})}, we remove the long-range dependencies selective modeling proposed in Section~\ref{sec:ssm}, \ie, we update node embeddings in Eq.~\eqref{eq:merge} with only $\mathbf{Z}_\text{MP}$. For \textbf{\modelname~(\textit{w/o PRI})}, we remove the PRI for DGSL regularizing term proposed in Section~\ref{sec:pri} by calculating the overall loss in Eq.~\eqref{eq:loss_all} with only the cross-entropy for link prediction term.

\subsubsection{Detailed Settings for Section~\ref{sec:exp_hyper}.}
We conducted a sensitivity analysis on four key hyperparameters. The value ranges for each hyperparameter are empirically selected, and the analysis is performed while keeping other hyperparameters and parameters fixed. The results are visualized with error bars to assess the impact of each hyperparameter on the task performance.

\subsubsection{Detailed Settings for Section~\ref{sec:exp_vis}.}
Given the large scale of the dynamic graph datasets used for evaluation, directly visualizing the dynamic graph structure is challenging and may not effectively convey useful information. Therefore, we indirectly analyze the learned spatio-temporal structural changes by visualizing the intra- and inter-graph structure weights before and after training. Note that, in the original dataset, there are no explicit edges between graph snapshots for the inter-graph structure. To mitigate this gap, we calculate the cosine similarity between the original embeddings of nodes from two secutive graph snapshots to represent the inter-graph structure weights before training.
\section{Additional Experiment Results}
\label{sec:more_results}
In this section, we provide additional experiment results.

\begin{figure*}[!t]
    \begin{minipage}{0.51\textwidth}
        \centering
        \includegraphics[height=3.05cm]{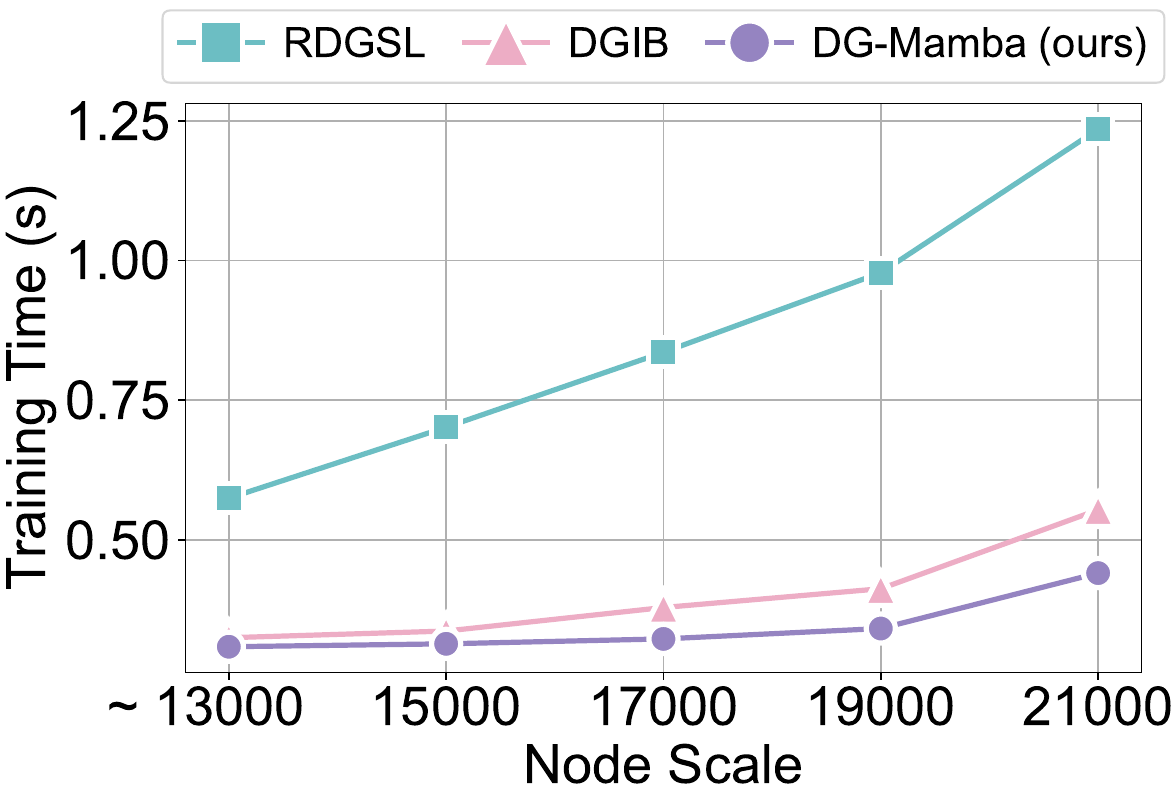}
        \hspace{-0.5em}
        \includegraphics[height=3.05cm]{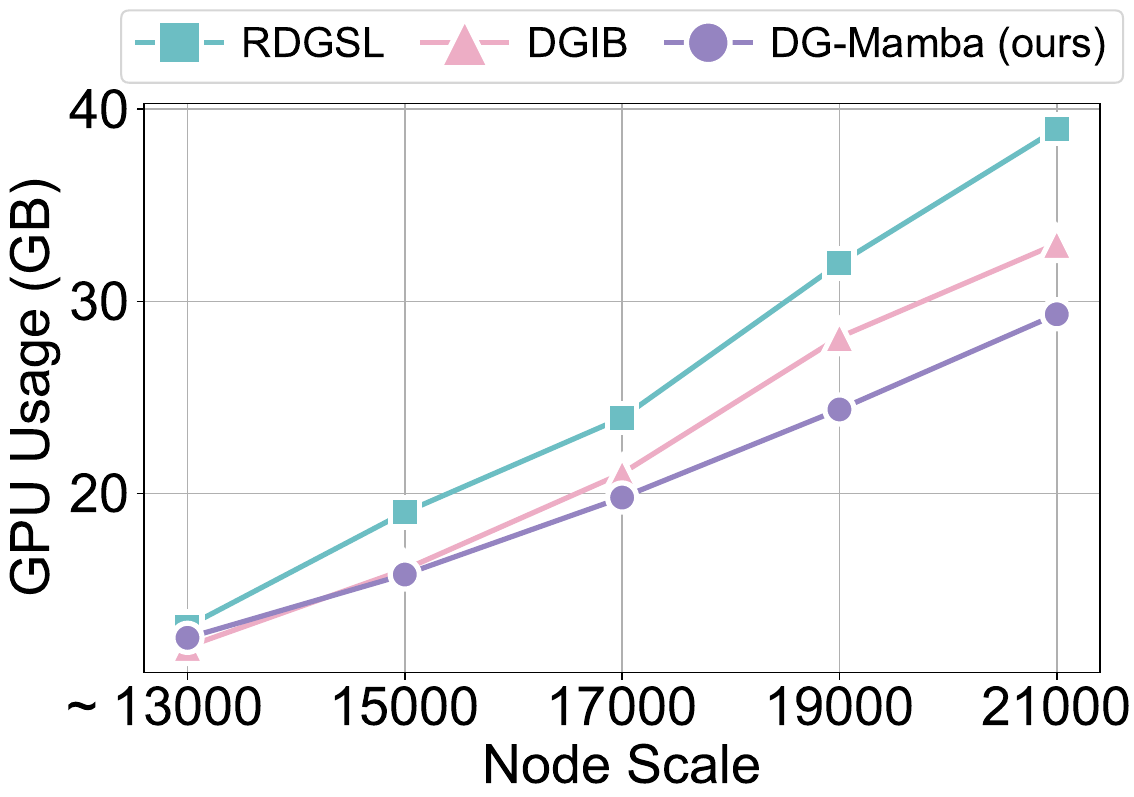}
        \caption{Node scaling efficiency evaluation on COLLAB.}
        \label{fig:effi_node_collab}
    \end{minipage}
    \hspace{-0.9em}
    \begin{minipage}{0.51\textwidth}
        \centering
        \includegraphics[height=3.05cm]{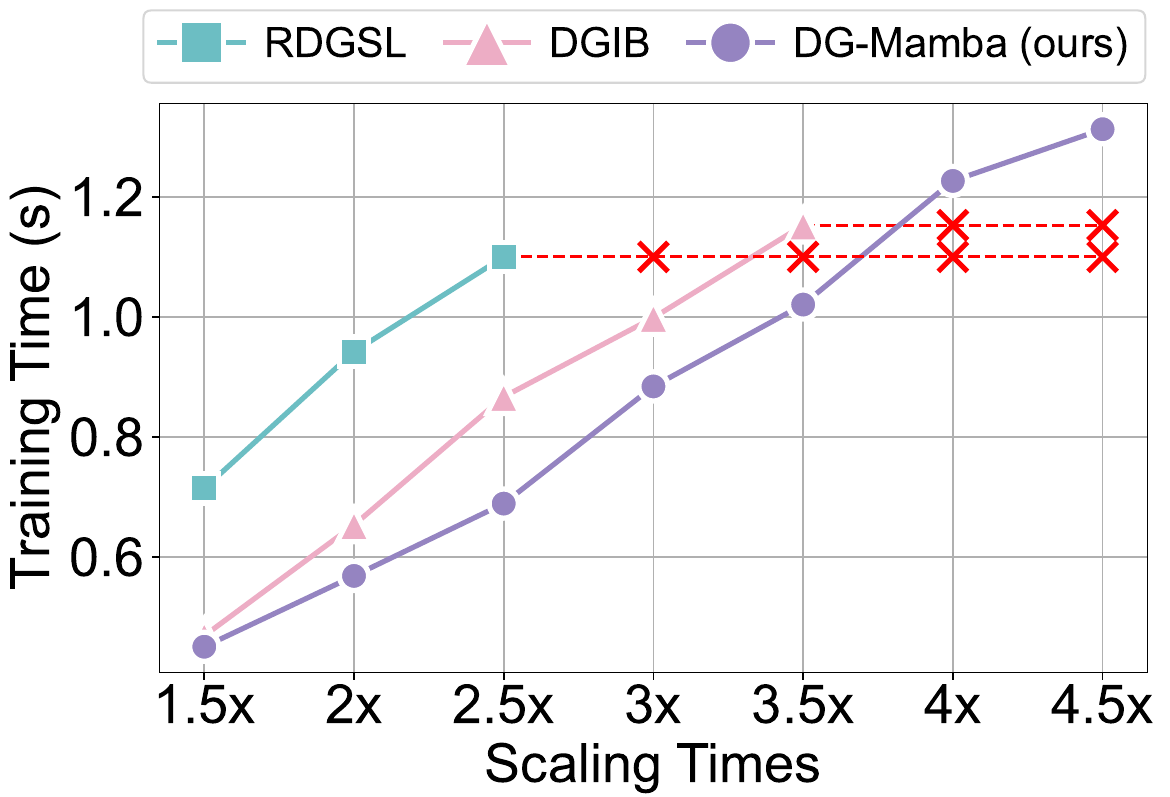}
        \hspace{-0.7em}
        \includegraphics[height=3.05cm]{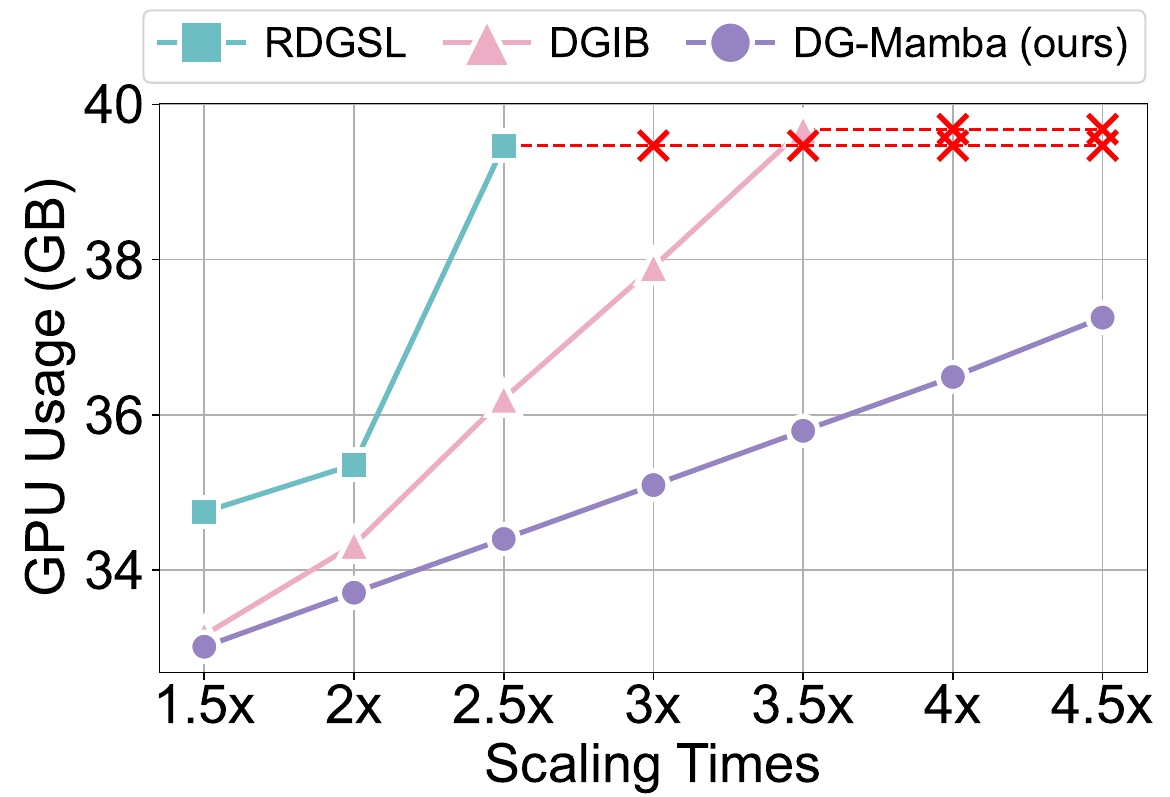}
        \caption{Sequence scaling efficiency evaluation on COLLAB.}
        \label{fig:effi_length_collab}
    \end{minipage}
\end{figure*}

\begin{figure*}[!t]
    \begin{minipage}{0.51\textwidth}
        \centering
        \includegraphics[height=3.05cm]{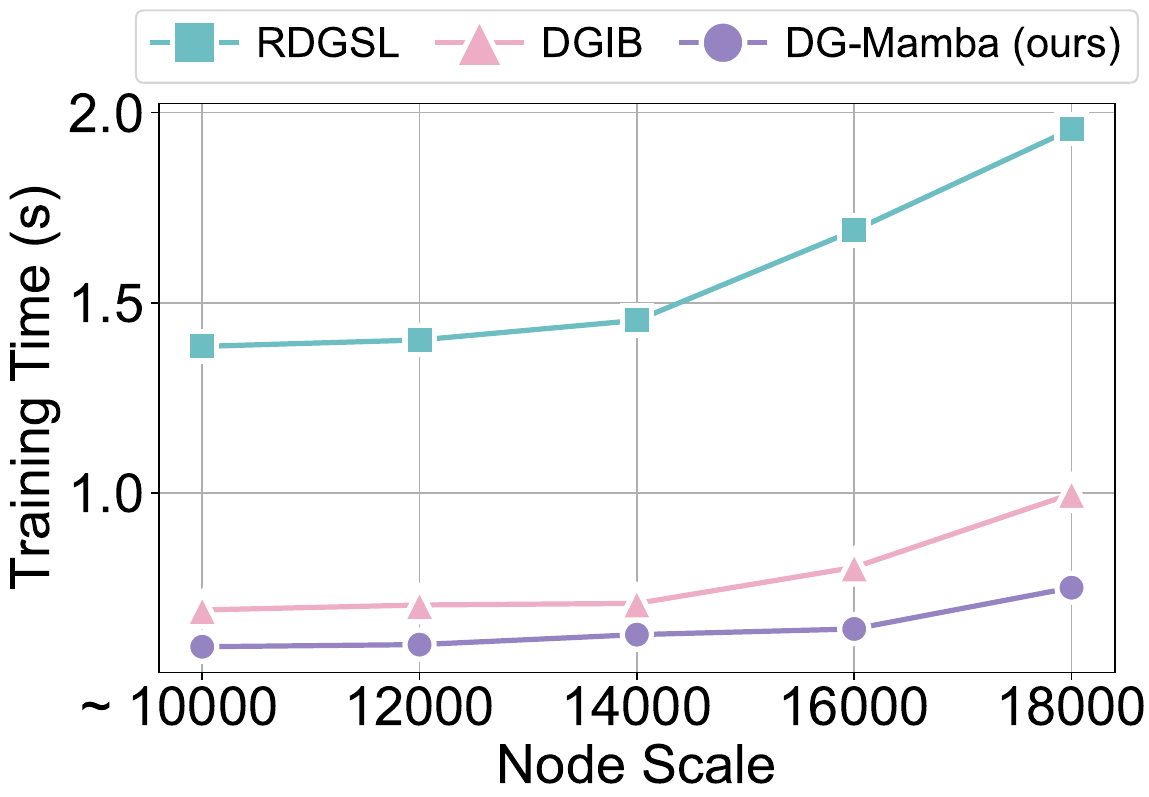}
        \hspace{-0.4em}
        \includegraphics[height=3.05cm]{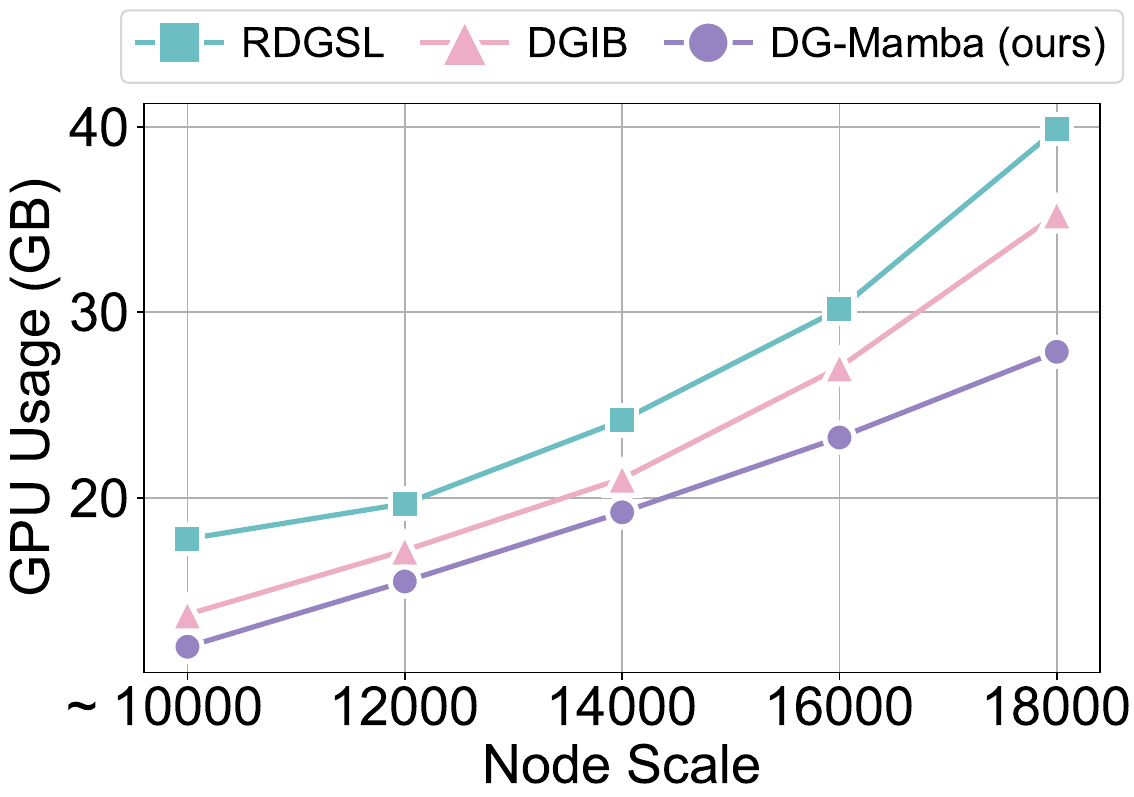}
        \caption{Node scaling efficiency evaluation on ACT.}
        \label{fig:effi_node_act}
    \end{minipage}
    \hspace{-0.9em}
    \begin{minipage}{0.51\textwidth}
        \centering
        \includegraphics[height=3.05cm]{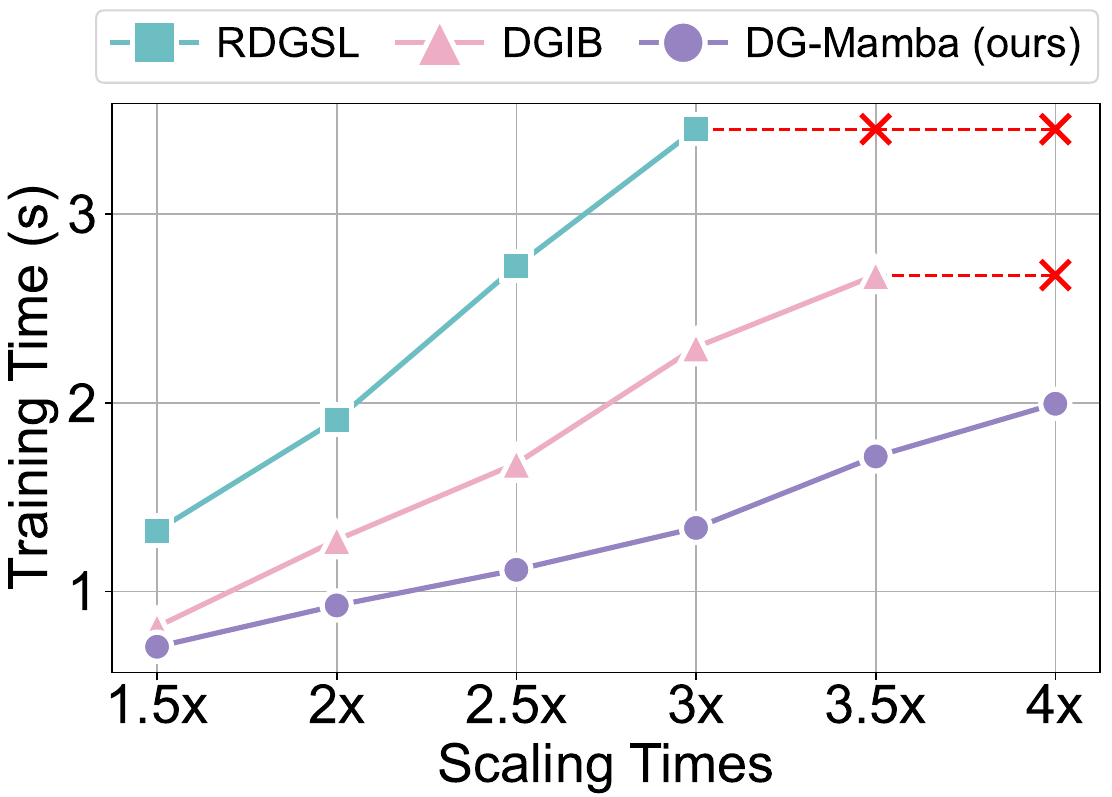}
        \hspace{-0.1em}
        \includegraphics[height=3.05cm]{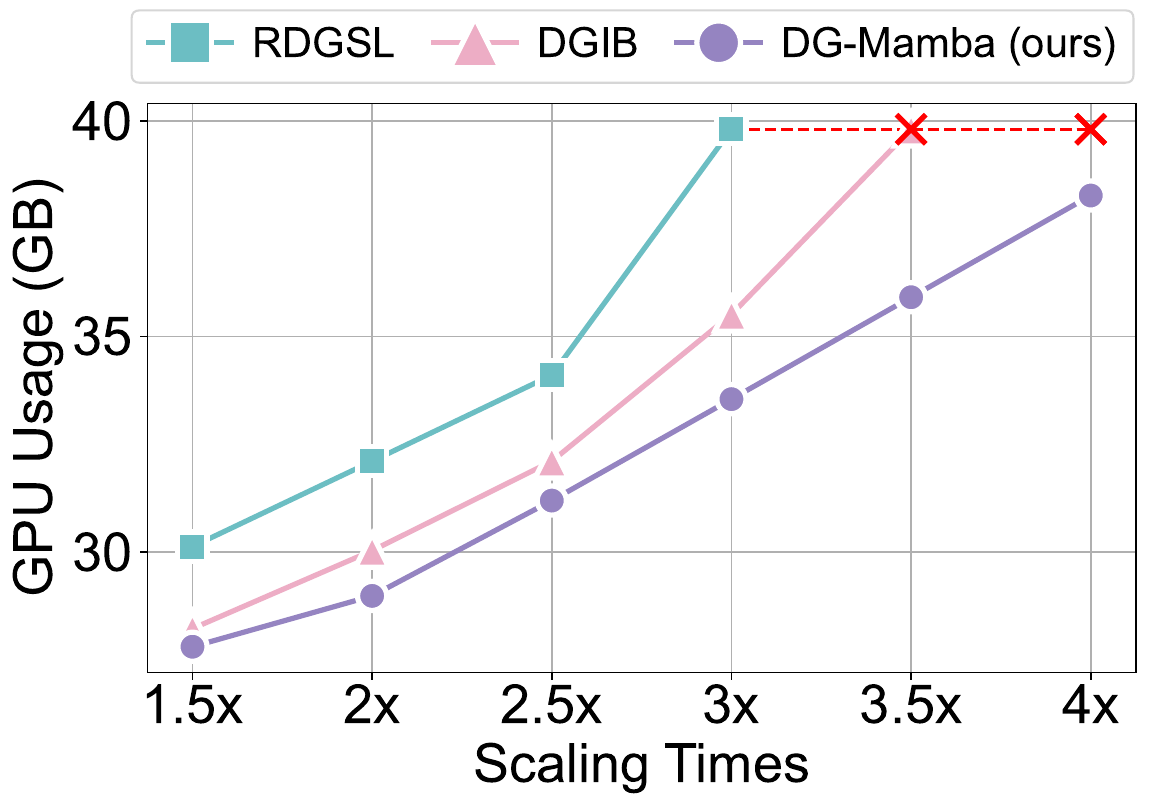}
        \caption{Sequence scaling efficiency evaluation on ACT.}
        \label{fig:effi_length_act}
    \end{minipage}
\end{figure*}

\begin{figure*}[!t]
\centering
\subfigure[COLLAB.]{
    \begin{minipage}[t]{0.33\linewidth}
    \centering
    \includegraphics[width=\linewidth]{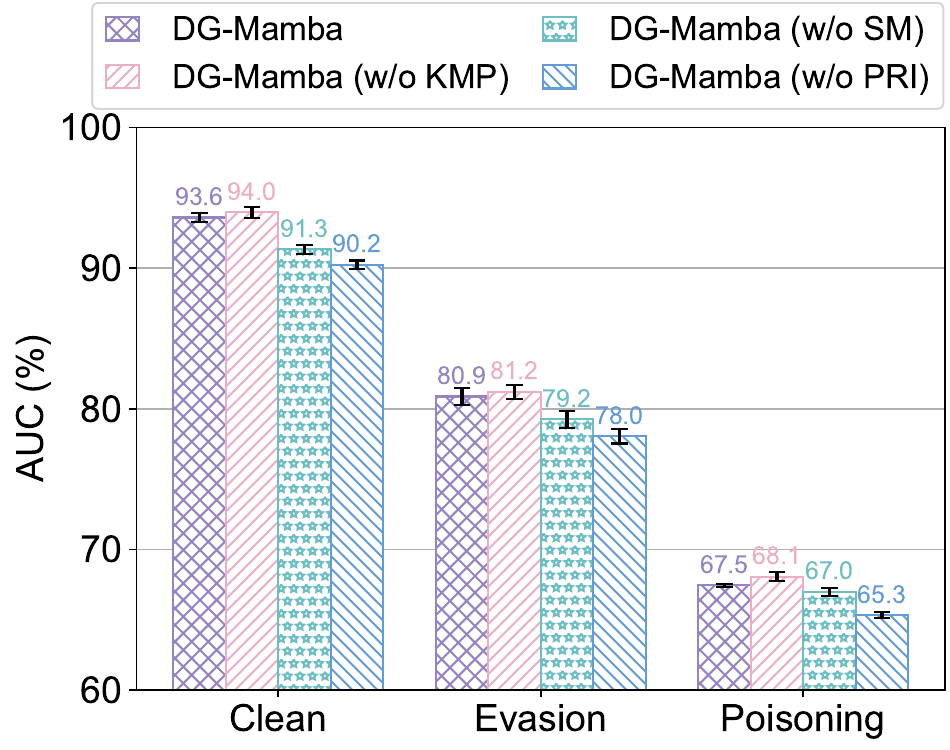}
    \label{fig:abla_collab_add}
    \vspace{-2cm}
    \end{minipage}
}
\hfill
\hspace{-1.2em}
\subfigure[Yelp.]{
    \begin{minipage}[t]{0.33\linewidth }
    \centering
    \includegraphics[width=\linewidth]{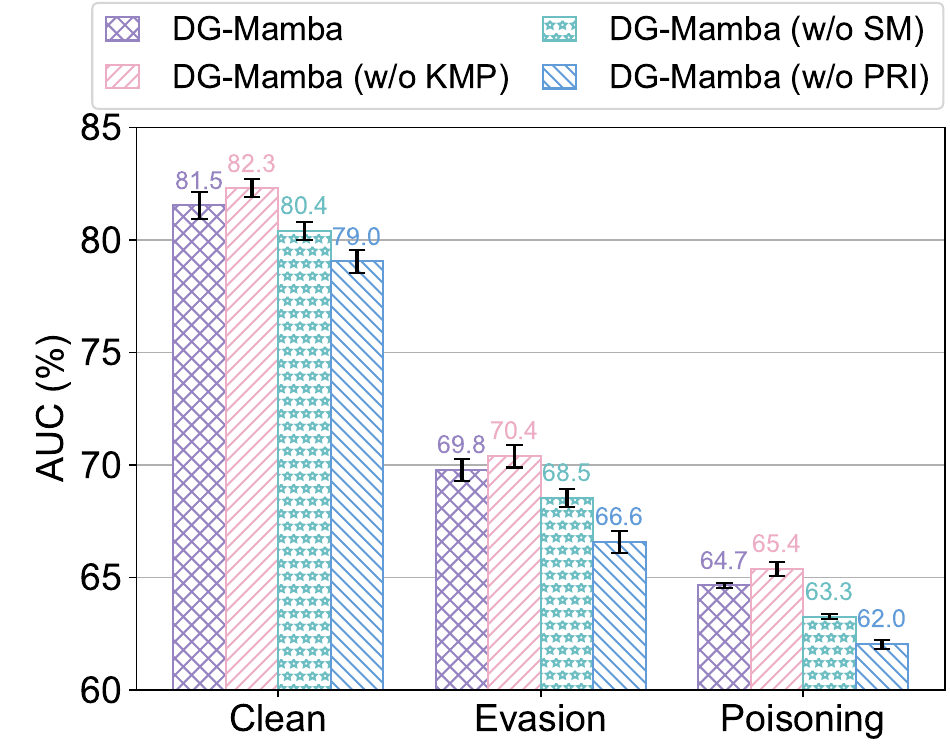}
    \label{fig:abla_yelp_add}
    \vspace{-2cm}
    \end{minipage}
}
\hfill
\hspace{-1.2em}
\subfigure[ACT.]{
    \begin{minipage}[t]{0.33\linewidth }
    \centering
    \includegraphics[width=\linewidth]{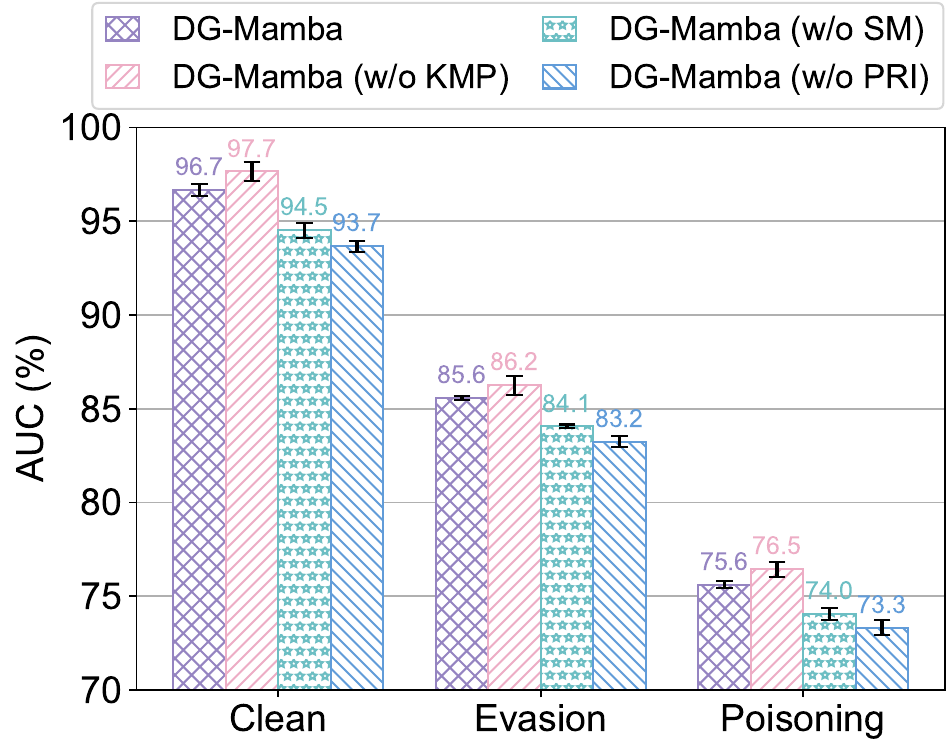}
    \label{fig:abla_act_add}
    \vspace{-2cm}
    \end{minipage}
}
\caption{Ablation studies on the clean, evasion, and poisoning datasets.}
\label{fig:abla_add}
\end{figure*}

\subsection{Additional Results: Scaling Efficiency Analysis}
\label{sec:scale_add}
We demonstrate additional results of scaling efficiency analysis for datasets COLLAB and ACT in Figure~\ref{fig:effi_node_collab} to Figure~\ref{fig:effi_length_act}, respectively. We provide additional analysis here.
\subsubsection{Analysis.}
The scaling efficiency analysis experiments focus on evaluating the efficiency of the \modelname~in terms of node scaling and sequence length scaling.
\begin{itemize}[leftmargin=*]
    \item \textbf{Node Scaling.} \modelname~demonstrates superior scalability. As node number increases, \modelname~shows near-linear growth in training time and GPU usage compared to other models. This indicates \modelname~handles larger graphs more efficiently, reducing the computational burden. Specifically, compared with RDGSL, \modelname~reduces training time by up to 65.1\% and GPU usage\footnote{We omit scenarios that cause OOM (Out-of-Memory).} by 24.8\% in COLLAB, and reduces training time by up to 62.0\% and GPU usage by 32.7\% in ACT.
    \item \textbf{Sequence Length Scaling.} \modelname~shows near-linear scalability. Compared to DGIB, \modelname~achieves a reduction in training time of up to 39.7\% and GPU usage by 12.8\% in COLLAB, and achieves a reduction in training time of up to 61.3\% and GPU usage by 15.7\% in ACT. Importantly, \modelname~continues to operate within GPU limits even when other baselines, \ie, RDGSL and DGIB, fail due to Out-Of-Memory (OOM) errors, particularly beyond a 6.5x scaling factor.
\end{itemize}

\begin{figure*}[!t]
    \centering
    \includegraphics[width=1\linewidth]{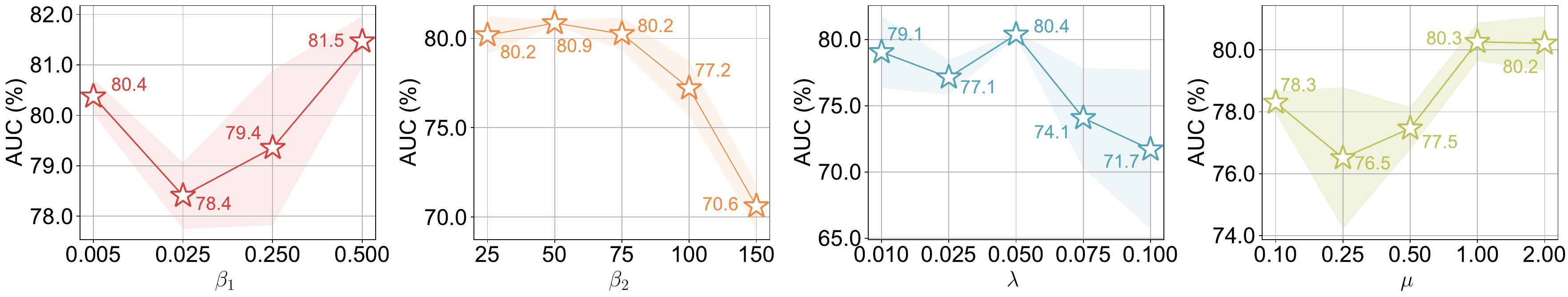}
    \caption{Hyperparameter sensitivity analysis on Yelp.}
    \label{fig:hyper_yelp}
\end{figure*}

\begin{figure*}[!t]
    \centering
    \includegraphics[width=1\linewidth]{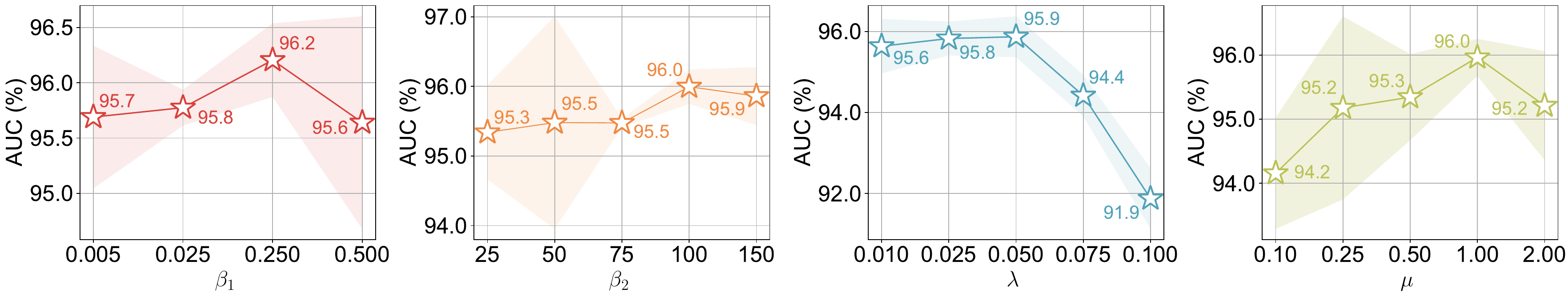}
    \caption{Hyperparameter sensitivity analysis on ACT.}
    \label{fig:hyper_act}
\end{figure*}

\subsection{Additional Results: Ablation Study}
We illustrate additional results of ablation studies on the clean, evasion, and poisoning datasets of COLLAB, Yelp, and ACT in Figure~\ref{fig:abla_add}. We provide additional analysis here.
\subsubsection{Analysis.}
\modelname~consistently performs well across all datasets and scenarios (clean, evasion, and poisoning). However, it is noteworthy that under clean conditions, removing the Kernelized Message-Passing (\textit{KMP}) results in a slight performance improvement. This suggests that \textit{KMP}, while enhancing efficiency, introduces some approximation that can slightly reduce the task performance.
\begin{itemize}[leftmargin=*]
    \item  \textbf{\modelname~(\textit{w/o KMP}).} Removing \textit{KMP} actually results in a slight increase in AUC across all datasets, as seen in the clean dataset scenarios where the AUC improves. For example, in the COLLAB dataset, the AUC increases from 93.6\% to 94.0\% after removing \textit{KMP}. This improvement occurs because \textit{KMP} is an approximation mechanism designed to boost efficiency at the potential cost of a small drop in precision. While \textit{KMP} is essential for handling larger datasets and improving scalability, its removal can enhance task performance in some cases.
    \item \textbf{\modelname~(\textit{w/o SM}).} Removing \textit{SM} causes a moderate performance drop. The impact is particularly evident in the COLLAB and ACT datasets, where the AUC drops by around 1\% to 3\% under all conditions. The results highlight that \textit{SM} plays a vital role in capturing long-range dependencies and improving model robustness, particularly in complex and adversarial environments.
    \item \textbf{\modelname~(\textit{w/o PRI}). }The absence of \textit{PRI} results in the most significant performance degradation across all scenarios, especially in the ACT dataset where AUC decreases by 2\% to 4\%. This suggests that \textit{PRI} for DGSL regularization is critical for ensuring that \modelname~focuses on the most relevant and least redundant structural information, enhancing both efficiency and robustness against adversarial attacks.
\end{itemize}
In conclusion, while \textit{KMP} enhances efficiency, its removal can lead to slight improvements in task performance, highlighting the trade-off between efficiency and precision. However, \textit{SM} and \textit{PRI} are indispensable for maintaining satisfying robustness and link prediction performance.

\begin{figure*}[!t]
    \begin{minipage}{0.51\textwidth}
        \centering
        \includegraphics[height=6.27cm]{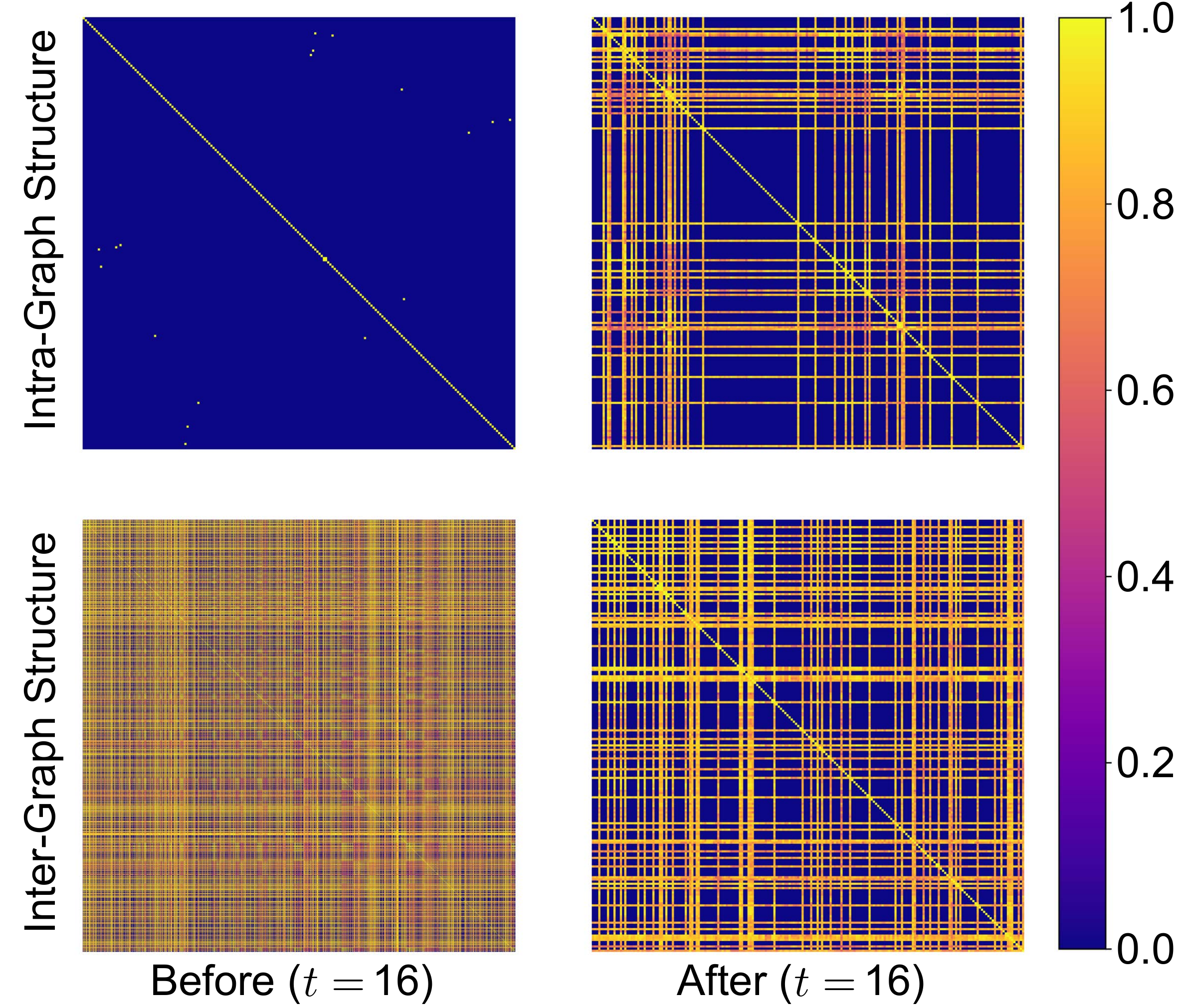}
        \caption{Structure visualization on COLLAB.}
        \label{fig:vis_collab}
    \end{minipage}
    \hspace{-0.9em}
    \begin{minipage}{0.51\textwidth}
        \centering
        \includegraphics[height=6.27cm]{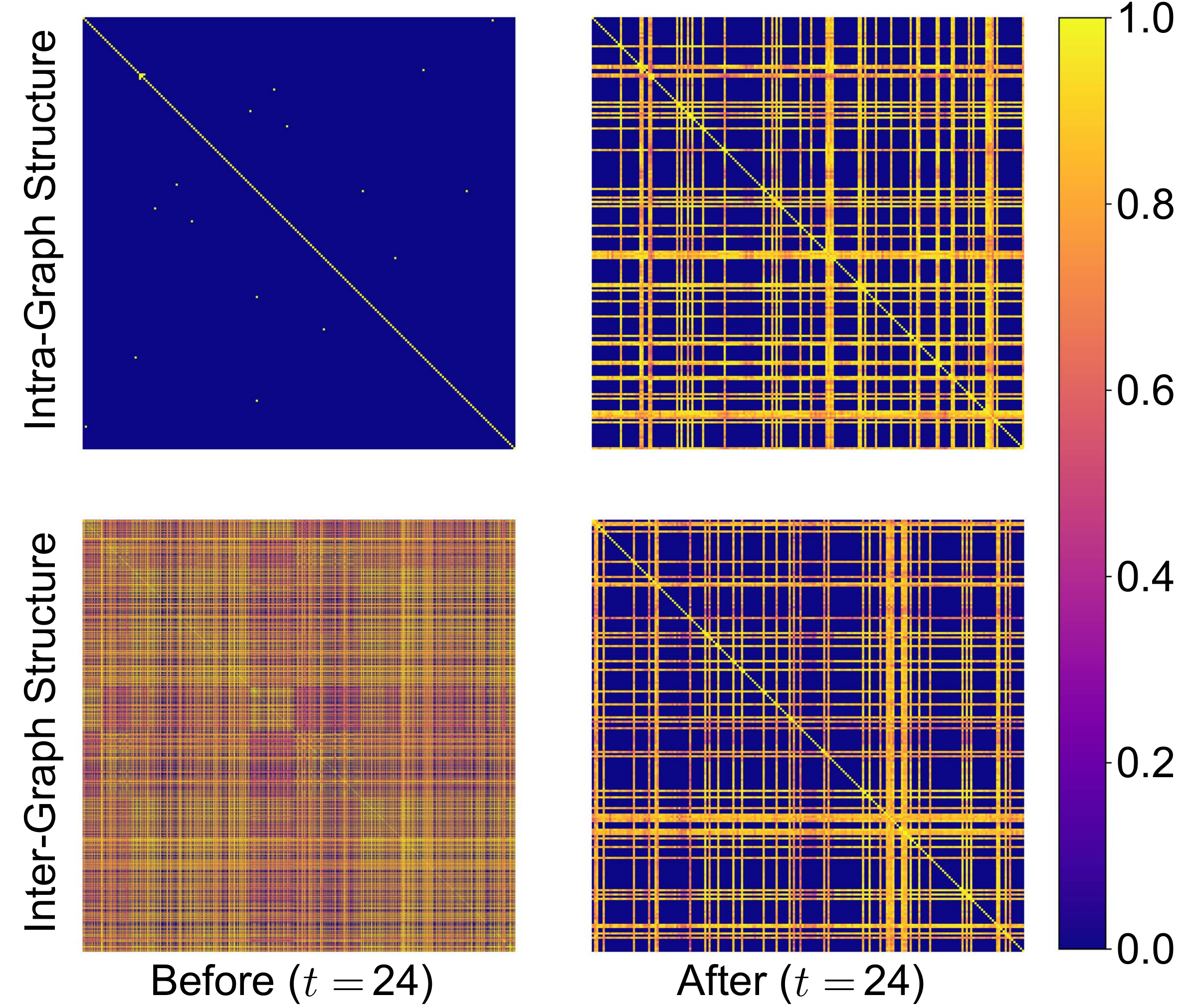}
        \caption{Structure visualization on Yelp.}
        \label{fig:vis_yelp}
    \end{minipage}
\end{figure*}

\subsection{Additional Results: Hyperparameter Sensitivity Analysis}
\label{sec:hyper_more}
We provide additional results of hyperparameter sensitivity analysis on Yelp and ACT in Figure~\ref{fig:hyper_yelp} and Figure~\ref{fig:hyper_act}, respectively. We provide additional analysis here.
\subsubsection{Analysis.}
Several key observations can be made.
\begin{itemize}[leftmargin=*]
    \item $\boldsymbol{\beta_1}$\textbf{: } controls the distortion in intra-graph structures. On both datasets, increasing $\beta_1$ generally improves performance up to a certain point. The optimal range for $\beta_1$ is around 0.250 to 0.500, where the AUC peaks. This suggests that moderate distortion helps in capturing more relevant structural information, but too much distortion can begin to harm performance.
    \item $\boldsymbol{\beta_2}$\textbf{: } controls the distortion in inter-graph structures. The analysis shows that while moderate values of $\beta_2$ (around 50 to 75) maintain good performance, higher values lead to a noticeable drop, especially on Yelp. This indicates that while some level of distortion can help in modeling temporal dependencies, too much can disrupt the temporal coherence, degrading the effectiveness.
    \item $\boldsymbol{\lambda}$\textbf{: } controls the trade-off between node embeddings and loss terms. The optimal performance is observed at $\lambda=$ 0.050 on both datasets. Higher values cause a significant drop in AUC, suggesting that \modelname~becomes overfitted to either the embeddings or the loss terms if $\lambda$ is not balanced correctly.
    \item $\boldsymbol{\mu}$\textbf{: } is the Lagrangian hyperparameter of PRI for DGSL. Increasing $\mu$ generally improves performance, with the best results at $\mu=$ 1.0. Beyond this point, the performance stabilizes or slightly decreases, indicating that while regularization is crucial, excessive regularization can slightly reduce the flexibility and performance of \modelname.
\end{itemize}
Overall, the sensitivity analysis highlights the importance of hyperparameter tuning to balance the robustness and task performance effectively. The results emphasize that while regularization and distortion are essential for capturing relevant structural features and ensuring robustness, there is a fine line where too much will degrade the performance.

\subsection{Additional Results: Visualization of Learned Structures}
We show additional results of visualization of the learned intra- and inter-graph structures on COLLAB and Yelp in Figure~\ref{fig:vis_collab} and Figure~\ref{fig:vis_yelp}, respectively. We provide additional analysis here.
\subsubsection{Analysis.}
The following observations can be made.
\begin{itemize}[leftmargin=*]
    \item \textbf{Intra-Graph Structures. }Before training, the intra-graph structures are relatively sparse, indicating that the initial connections between nodes are limited and potentially weak in capturing the underlying predictive patterns. After training, through dynamic graph structure learning, these intra-graph connections are significantly strengthened, particularly in areas that represent underlying predictive patterns and spatio-temporally invariant structures. This enhancement suggests that \modelname~has effectively identified and reinforced key structures that are essential for accurate prediction, while preserving important temporal dependencies.
    \item \textbf{Inter-Graph Structures. }Before training, the inter-graph structures are initially dense, reflecting the raw similarity calculations from the features without any sparsification. This density indicates that the initial connections between different graph snapshots are numerous, potentially including a lot of redundant information. After training, the inter-graph structures exhibit a form of sparsification. The dynamic graph structure learning process has pruned away redundant connections and strengthened the critical structures that play a key role in spatio-temporal message-passing. This pruning not only reduces complexity but also enhances the robustness of the learned graph representations by focusing on the most relevant and informative connections.
\end{itemize}

The visualizations effectively demonstrate \modelname's capability to refine both intra- and inter-graph structures, reinforcing essential connections while eliminating redundancies, which contributes to more robust graph representations.

\subsection{Additional Analysis: Against Non-Targeted Adversarial Attacks}
We provide additional insights for the results in Section~\ref{sec:exp_non_targeted}.
\begin{itemize} [leftmargin=*]
    \item \textbf{Superior Performance.} Our \modelname~consistently outperforms other baselines across all datasets. This superiority is especially noticeable in scenarios involving structure attacks and feature attacks at varying intensity levels. For example, on the COLLAB dataset, \modelname~achieves an AUC of 92.60\% under structure attacks, while the next best-performing DGIB-Cat reaches only 84.16\%. Similarly, under feature attacks ($\rho=$ 0.5), \modelname~attains 68.53\% AUC, which significantly outperforms other baselines like RGCN and WinGNN.
    \item \textbf{Compared with Dynamic GNNs and DGSL Baselines.} Dynamic GNNs, like DySAT and EvolveGCN, demonstrate underwhelming performance under adversarial conditions due to their insufficient modeling of complex coupling dynamics in the graph data. Similarly, DGSL baselines like RDGSL and TGSL show sensitivity to noise and adversarial attacks, which results in sharp declines in AUC scores, particularly under high-intensity feature attacks.
    \item \textbf{Compared with Robustness-based Baselines.} While DGIB shows comparable robustness in some cases, they generally fall short when handling long-range dependencies and complex temporal relationships. \modelname's selective dependencies modeling and robust graph structure learning contribute to its better performance and lower sensitivity to adversarial attacks.
\end{itemize}

In conclusion, results demonstrate that \modelname~is superior robust against non-targeted adversarial attacks, outperforming other baseline models significantly. Its advanced structure learning and selective state-space modeling contribute to maintaining high performance under challenging adversarial conditions.

\subsection{Additional Analysis: Against Targeted Adversarial Attacks}
We provide additional insights for the results in Section~\ref{sec:exp_targeted}.
\begin{itemize}[leftmargin=*]
    \item \textbf{Robustness.} \modelname~demonstrates strong robustness across all three datasets under both evasion and poisoning attacks. \modelname~shows the lowest average percentage decrease in AUC scores compared to other competitive baselines, highlighting its resilience against targeted manipulations. For evasion attacks, \modelname's AUC drops from 93.60\% to 81.78\% in COLLAB when the number of perturbations ($n$) is 1, resulting in a relative decrease of 12.6\%. This decrease is less severe than that observed in other baselines. For poisoning attacks, \modelname again outperforms its peers. For example, AUC drop for \modelname~is 13.1\% under evasion attacks with $n =$ 1 in Yelp, while for TGSL, the decrease is much more pronounced.
    \item \textbf{Baseline Comparisons.} Baselines like WinGNN and DGIB-Cat, which also demonstrate relative robustness, still suffer from more significant performance drops compared to \modelname. For example, WinGNN’s AUC decreases by 19.6\% under evasion attacks with $n =$ 2 in ACT, whereas \modelname's decrease is only 11.5\%. It is noteworthy that, ACT poses significant challenges for most baselines, which is characterized by strong temporal relations and complex interactions. Baselines like SpoT-Mamba and TGSL exhibit larger AUC decreases across the board, indicating their higher vulnerability to targeted adversarial attacks. This contrasts with \modelname, which manages to retain a higher level of accuracy even under intense adversarial conditions.
    \item \textbf{Gradual Degradation.} \modelname~shows a more gradual degradation in the performance as the intensity of the adversarial attack increases. This indicates that \modelname~is not only robust to initial perturbations but can also withstand higher levels of targeted manipulation better than other baselines.
\end{itemize}
In summary, results show \modelname~is superior against targeted adversarial attacks, outperforming other state-of-the-art baselines in terms of maintaining accuracy and robustness. This robustness across various datasets and attack scenarios highlights \modelname's potential for real-world applications where security and reliability are paramount.
\section{Implementation Details}
In this section, we provide implementation details.

\subsection{Hardware-Aware Dynamic Graph Selective Scan Implementations}
\label{sec:hardware}
Implementation of the Hardware-Aware Dynamic Graph Selective Scan Algorithm ($\operatorname{DGSS}$) for capturing the long-range dependencies in dynamic graphs necessitates efficient management of the memory and computational demands, particularly when processing large-scale dynamic graphs. Below, we detail the implementations, with a focus on minimizing memory overhead and maximizing computational throughput through hardware-specific techniques.

Specifically, \modelname~reads the $\mathcal{O}(BTD + ND)$ input data of $\boldsymbol{\mathcal{A}}$, $\boldsymbol{\mathcal{B}}$, $\boldsymbol{\mathcal{C}}$, and $\boldsymbol{\Delta}$ from High Bandwidth Memory (HBM), discretizes the inter-graph structures, and processes the intermediate stages sized $\mathcal{O}(BTDN)$ in Static Random-Access Memory (SRAM), before dispatching the final output, sized $\mathcal{O}(BTD)$, back to HBM. The discretization step focuses on isolating the critical components that contribute to long-range dependencies, effectively filtering out noise and less relevant temporal structures. Within the SRAM, a parallel associative scan is executed, computing intermediate states essential for modeling long-range dependencies within the temporal dimension of the dynamic graph. By localizing computation within SRAM, data is processed more rapidly, significantly reducing the need for extensive memory access and thereby enhancing throughput.

This implementation substantially decreases the number of memory read/write operations, reducing I/O costs by a factor of $\mathcal{O}(N)$ and achieving a practical speedup. The fusion of multiple operations into a single kernel not only accelerates processing but also preserves the integrity of long-range dependencies within the dynamic graph structures. For sequences that are extremely long to fit entirely within SRAM, a chunk-based processing strategy is employed. The sequence is divided into smaller chunks, each undergoing the same processing steps independently. Intermediate states from one chunk inform the next, ensuring long-range dependencies are maintained across the entire sequence. During the backward pass, memory demands are further mitigated using recomputation techniques. Instead of storing intermediate states of size $\mathcal{O}(BTDN)$, which could impose memory constraints, these states are recomputed on the fly during the backward computation. This approach significantly reduces the memory required in HBM, avoiding the overhead associated with storing and reloading large datasets.

By integrating these hardware-aware optimizations into the selective modeling process, \modelname~is enhanced in its ability to capture long-range dependencies in dynamic graphs, while ensuring the system remains efficient and scalable for large-scale applications.

\subsection{Implementation Details of \modelname}
\subsubsection{Training.}
The number of training epochs for optimizing our proposed \modelname~and all baselines is set to 1000. We employ an early stopping strategy, terminating the training if there is no improvement in the validation set performance for 50 consecutive epochs. For \modelname, the hyperparameter $\beta_1$ is chosen from \{0.005, 0.025, 0.250, 0.500\}, $\beta_2$ is chosen from \{25, 50, 75, 100, 150\}, $\lambda$ is chosen from \{0.010, 0.025, 0.050, 0.075, 0.100\}, and $\mu$ is chosen from \{0.10, 0.5, 0.50, 1.00, 2.00\}. For other parameters, we utilize the Adam optimizer~\cite{kingma2014adam}, carefully tuning the learning rate and weight decay specific to each dataset. A grid search is performed on the validation set to identify the optimal parameter settings. All parameters are initialized randomly.

\subsubsection{Evaluation.}
\label{sec:evaluation_add}
In line with experimental settings, we randomly split the dynamic graph datasets into training, validation, and testing sets in chronological order. Negative edges are sampled from node pairs that do not have existing edges, with the same negative edges used for both validation and testing across all baseline methods and our model. The number of positive edges is set equal to the number of negative edges. We evaluate performance using the Area Under the ROC Curve (AUC). As our focus is on future link prediction, we use the inner product of a pair of learned node representations to predict link occurrence. Specifically, the link predictor $g(\cdot)$ is implemented as the inner product of hidden embeddings, a common approach in future link prediction tasks. The non-linear activation function $\sigma$ is set to $\mathrm{Sigmoid(\cdot)}$. All experiments are run five times, and we report the average results along with standard deviations.

\subsection{Implementation Details of Baselines}
We provide the baseline methods implementations with respective licenses as follows.
\begin{itemize}[leftmargin=*]
    \item \textbf{GAE}~\cite{kipf2016variational}: \texttt{[MIT License]}\\
    \url{https://github.com/DaehanKim/vgae_pytorch}.
    \item \textbf{VGAE}~\cite{kipf2016variational}: \texttt{[MIT License]}\\
    \url{https://github.com/DaehanKim/vgae_pytorch}.
    \item \textbf{GAT}~\cite{velivckovic2018graph}: \texttt{[MIT License]}\\
    \url{https://github.com/pyg-team/pytorch_geometric}.
    \item \textbf{GCRN}~\cite{seo2018structured}: \texttt{[MIT License]}\\
    \url{https://github.com/youngjoo-epfl/gconvRNN}.
    \item \textbf{EvolveGCN}~\cite{pareja2020evolvegcn}: \texttt{[Apache License]}\\
    \url{https://github.com/IBM/EvolveGCN}.
    \item \textbf{DySAT}~\cite{sankar2020dysat}: \texttt{[Unspecified]}\\
    \url{https://github.com/FeiGSSS/DySAT_pytorch}.
    \item \textbf{SpoT-Mamba}~\cite{choi2024spot}: \texttt{[Unspecified]}\\
    \url{https://github.com/bdi-lab/SpoT-Mamba}.
    \item \textbf{RDGSL}~\cite{zhang2023rdgsl}: \texttt{[Unspecified]}\\
    \url{https://github.com/FDUDSDE/RDGSL}.
    \item \textbf{TGSL}~\cite{zhang2023time}: \texttt{[MIT License]}\\
    \url{https://github.com/ViktorAxelsen/TGSL}.
    \item \textbf{RGCN}~\cite{zhu2019robust}: \texttt{[Unspecified]}\\
    \url{https://github.com/ZW-ZHANG/RobustGCN}.
    \item \textbf{WinGNN}~\cite{zhu2023wingnn}: \texttt{[Unspecified]}\\
    \url{https://github.com/thudm/WinGNN}.
    \item \textbf{DGIB}~\cite{yuan2024dynamic}: \texttt{[MIT License]}\\
    \url{https://github.com/RingBDStack/DGIB}.
\end{itemize}

\subsection{Hardware and Software Configurations}
\label{sec:config}
We conduct the experiments with:
\begin{itemize}[leftmargin=*]
    \item Operating System: Ubuntu 20.04 LTS.
    \item CPU: Intel(R) Xeon(R) Platinum 8358 CPU@2.60GHz with 1TB DDR4 of Memory.
    \item GPU: NVIDIA Tesla A100 SMX4 with 40GB of Memory.
    \item Software: CUDA 10.1, Python 3.8.12, PyTorch\footnote{\url{https://github.com/pytorch/pytorch}} 1.9.1, PyTorch Geometric\footnote{\url{https://github.com/pyg-team/pytorch_geometric}} 2.0.1.
\end{itemize}
\section{Related Work Discussions}
\label{sec:related_work_add}
\subsection{Dynamic Graph Representation Learning}
Dynamic graph learning is prevalent across various domains. Based on the types of temporal information to be encoded, dynamic graphs can be classified as either discrete-time or continuous-time~\cite{kazemi2020representation}. Discrete-time dynamic graphs typically consist of discrete snapshots that capture periodic changes, while lacking connection patterns between graphs. Continuous-time dynamic graphs retain all edges within a single graph, marked with timestamps, which allows for the efficient storage of the entire dynamic graph but requires sophisticated encoding of temporal information.

\textbf{Dynamic Graph Neural Networks (DGNNs)} are widely used to learn dynamic graph representations by inherently modeling both spatial and temporal predictive patterns. Existing DGNNs can be categorized into: \ding{172} \textbf{Stacked DGNNs} apply separate GNNs to process each graph, and feed the GNNs output into deep sequential models. Stacked DGNNs capture dynamics alternately, making them the mainstream. \ding{173} \textbf{Integrated DGNNs} serve as the encoders that integrate GNNs and deep sequential models within a single layer, unifying both spatial and temporal modeling~\cite{han2021dynamic}. The essence of both is to encode the temporal dynamics into learnable node representations for downstream applications.

Compared to the framework in our paper, where we introduce \modelname~with a focus on enhancing both efficiency and robustness, most traditional DGNN methods do not address the computational complexity and robustness simultaneously, especially in the presence of adversarial attacks.

\subsection{Robust Dynamic Graph Learning}
Dynamic graphs often originate from open environments, which naturally contain noise and redundant features irrelevant to the prediction task. This compromises DGNN performance in downstream tasks. Additionally, DGNNs are prone to the over-smoothing phenomenon, making them less robust and vulnerable to perturbations and adversarial attacks~\cite{zhu2023wingnn}.
Compared to robust GNNs for static graphs, which can be generally grouped into adversarial training~\cite{xu2019topology, feng2019graph}, graph denoising~\cite{wu2019adversarial, zhang2020gnnguard}, and certifiable robustness~\cite{wang2021certified, zugner2019certifiable}, there are no DGNNs tailored for efficient robust representation learning, currently.

In contrast, our \modelname~framework explicitly addresses these challenges by formulating the self-supervised Principle of Relevant Information (PRI) to dynamic graph structure learning, which regularizes the learned structures to enhance global robustness, a feature largely absent in existing DGNN literature.

\subsection{Dynamic Graph Structure Learning}
Graph Structure Learning (GSL) has gained significant attention in recent years, aiming to simultaneously learn a denoised graph structure and better graph representations~\cite{zhu2021deep}. It is worth noting that, GSL has also been widely used to enhance the robustness of GNNs. Currently, many studies have successfully explored GSL methods for static graphs. However, for dynamic graphs that incorporate both spatial and temporal structures, Dynamic Graph Structure Learning (DGSL) remains largely under-explored.

SLCNN~\cite{zhang2020spatio} explicitly models spatio-temporal structure with Pseudo 3D networks for traffic forecasting.
ST-LGSL~\cite{tang2022spatio} learns implicit structures with the respective diffusion and gated temporal convolutions.
TGSL~\cite{zhang2023time} enhances temporal structure learning by predicting and selecting time-aware edges with Gumbel-Top-$k$ sampling.
DGLL~\cite{li2023dynamic} forecasts multivariate time series by decomposing association patterns into long-short-term structure dependencies.
RDGSL~\cite{zhang2023rdgsl} introduces a dynamic graph filter to effectively denoise graphs and enhance representation robustness.
DHSL~\cite{zhang2018dynamic} dynamically optimizes hypergraph structures using correlations from both label and feature spaces.
DGIB~\cite{yuan2024dynamic} applies the Information Bottleneck to refine dynamic structures and achieve robustness against adversarial attacks.

Overall, DGSL faces the additional challenge of simultaneously learning structures within individual graphs and optimizing message-passing structures across graph snapshots. The enhanced spatio-temporal coupling structure is expected to improve the robustness of the representations. Another significant challenge is the computational efficiency bottleneck, as existing methods exhibit quadratic complexity in both the spatial and temporal dimensions, rendering them impractical for large-scale and long-sequence dynamic graphs. Our \modelname~introduces a novel kernelized dynamic message-passing operator, which significantly reduces computational complexity, addressing the efficiency bottleneck inherent in existing DGSL methods.

\subsection{State Space Models}
State Space Models (SSMs) have long been a fundamental framework in the field of dynamic system modeling, extensively utilized in areas ranging from control theory~\cite{aoki2013state} to deep learning~\cite{gu2023mamba}. The essence of the SSM is known as linear time-invariant systems that map input sequence to response sequence. To make SSMs differentiable, discrete-time SSMs utilize a step size parameter~\cite{gu2020hippo} and can be computed efficiently. Further, Structured State Space Models (S4) enhance the efficiency and scalability of SSMs by employing reparameterization techniques. Recently, \citet{gu2023mamba} introduced the efficient and powerful Selective SSM (SSSM) called Mamba, which uses recurrent scans and a selection mechanism to act as the attention mechanism. SSMs and their variants show promising performance in efficiently modeling sequential data, while also demonstrating their capabilities in graph learning~\cite{behrouz2024graph, wang2024graph, li2024stg, choi2024spot}. However, while SSMs and their variants can efficiently model sequential data with vector-form features, there is no consensus or in-depth exploration on how to model the non-Euclidean dynamic structures and matrix-form node features for a dynamic graph.

Our \modelname~incorporates the Selective SSM (SSSM) with the discretized inter-graph structures, not only captures long-range dependencies but also addresses the computational challenges, presenting a comprehensive solution for dynamic graph representation learning that surpasses the existing SSM applications.
\section{Limitations}
The limitations of \modelname~include the following aspects. First, while we evaluated the robustness of \modelname~against various adversarial attacks, these scenarios primarily focused on altering graph structures and node features. Other types of adversarial attacks, such as those targeting edge attributes or introducing synthetic nodes, have not been thoroughly explored. Additionally, although \modelname~improves robustness and efficiency, the interpretability of the learned dynamic graph structures still remains challenging. The complexity may obscure the reasons behind certain predictions or the importance of specific graph subgraphs or motifs. Lastly, the proposed \modelname~is currently tailored only for discrete dynamic graphs and has not been validated on continuous dynamic graphs, which is another domain of research scope. In summary, future research could extend the \modelname's applicability, explore a wider range of adversarial attack scenarios, enhance the interpretability of learned structures, and verify its effectiveness on continuous dynamic graphs.

\end{document}